\documentclass[11pt]{article}
\usepackage{hyperref,fullpage}
\usepackage[square]{natbib}

\usepackage{etoolbox}

\makeatletter

\newcommand{\define}[4][ignore]{%
  \ifstrequal{#1}{ignore}{}{
  \@namedef{thmtitle@#2}{#1}}%
  \@namedef{thm@#2}{#4}%
  \@namedef{thmtypen@#2}{lemma}%
  \newtheorem{thmtype@#2}[theorem]{#3}%
  \newtheorem*{thmtypealt@#2}{#3~\ref{#2}}%
}

\newcommand{\state}[1]{%
  \@namedef{curthm}{#1}
  \@ifundefined{thmtitle@#1}{
  \begin{thmtype@#1}
    }{
  \begin{thmtype@#1}[\@nameuse{thmtitle@#1}]
  }
    \label{#1}
    \@nameuse{thm@#1}
  \end{thmtype@#1}
  \@ifundefined{thmdone@#1}{
  \@namedef{thmdone@#1}{stated}%
  }{}
}

\newcommand{\restate}[1]{%
  \@namedef{curthm}{#1}
  \@ifundefined{thmtitle@#1}{
    \begin{thmtypealt@#1}
    }{
  \begin{thmtypealt@#1}[\@nameuse{thmtitle@#1}]
  }
    \@nameuse{thm@#1}
  \end{thmtypealt@#1}
  \@ifundefined{thmdone@#1}{
  \@namedef{thmdone@#1}{stated}%
  }{}
}

\newcommand{\thmlabel}[1]{
  \@ifundefined{thmdone@\@nameuse{curthm}}{\label{#1}
    }{\tag*{\eqref{#1}}}
}
\makeatother

\newcommand{\D}[1]{\Gamma_{#1}}
\newcommand{\DD}[2]{\Gamma_{#1}(#2)}

\newcommand{\Ev}{\mathcal{E}}
\newcommand{\Count}{\textsc{Count}}
\newcommand{\NCount}{\textsc{NCount}}
\newcommand{\NoLatents}{\textsc{NoLatents}\xspace}
\newcommand{\GenRecovery}{$\alpha$-\textsc{General}\xspace}
\newcommand{\IdentifyNbr}{\textsc{Identify-OutNbr}\xspace}
\newcommand{\Identifybidir}{\textsc{Identify-Bidirected}\xspace}
\newcommand{\BoundedDegree}{$\alpha$-\textsc{BoundedDegree}\xspace}
\newcommand{\abBoundedDegree}{$(\alpha, \beta)$-\textsc{BoundedDegree}\xspace}

\newcommand{\Recovery}{$(\alpha,\beta)$-\textsc{Recovery}\xspace}
\newcommand{\RecoverG}{\textsc{RecoverG}\xspace}

\newcommand{\ALG}{\text{ALG}\xspace}
\newcommand{\out}[2]{\text{ch}_{#1}(#2)}
\newcommand{\inc}[2]{\text{pa}_{#1}(#2)}
\newcommand{\bi}[2]{\text{sp}_{#1}(#2)}
\newcommand{\NN}[2]{N_{#1}(#2)}

\def \doo {\mathrm{do}}
\def \diff {\mathrm{diff}}

\def \itr {\mathrm{itr}}

\def \UUU {\mathcal{U}}

\usepackage{algorithmic}
\usepackage{algorithm}
\usepackage{soul}

\usepackage{tikz}

\usetikzlibrary{matrix,decorations.pathreplacing}

\usepackage{xcolor,times}

\usepackage{easybmat}
\usepackage{multirow,bigdelim}
\usepackage{amsmath}
\usepackage{amsthm}
\usepackage{amssymb}
\usepackage{fancybox,mathtools}
\usepackage{tcolorbox}
\usepackage{subcaption}
\usepackage[english]{babel}
\usepackage{graphicx}
\usepackage{soul}
\usepackage{wrapfig,epsfig}
\usepackage{epstopdf}
\usepackage{url}
\usepackage{graphicx}
\usepackage{color}
\usepackage{epstopdf}
\usepackage{booktabs}
\usepackage{multirow}
\usepackage{siunitx}

\usepackage{scrextend}
\usepackage[T1]{fontenc}
\usepackage{bbm}
\usepackage{xspace}

\usepackage{tikz}
\usepackage{hyperref}  
\hypersetup{colorlinks=true,citecolor=blue,linkcolor=blue} 
\usetikzlibrary{arrows}
\newcommand{\indep}{\rotatebox[origin=c]{90}{$\models$}}
\newcommand*{\centernot}{%
  \mathpalette\@centernot
}
\newcommand{\notindep}{\not\!\perp\!\!\!\perp}

\newenvironment{CompactEnumerate}{
\begin{list}{\roman{enumi}.}{%
\usecounter{enumi}
\setlength{\leftmargin}{12pt}
\setlength{\itemindent}{5pt}
\setlength{\topsep}{1pt}
\setlength{\itemsep}{-3pt}
}}
{\end{list}}

\def\@centernot#1#2{%
  \mathrel{%
    \rlap{%
      \settowidth\dimen@{$\m@th#1{#2}$}%
      \kern.5\dimen@
      \settowidth\dimen@{$\m@th#1=$}%
      \kern-.5\dimen@
      $\m@th#1\not$%
    }%
    {#2}%
  }%
}

\newtheorem{theorem}{Theorem}[section]
\newtheorem{lemma}[theorem]{Lemma}
\newtheorem{definition}[theorem]{Definition}

\newtheorem{proposition}[theorem]{Proposition}
\newtheorem{corollary}[theorem]{Corollary}

\newtheorem{claim}[theorem]{Claim}

\newcommand{\wh}{\widehat}

\renewcommand{\varepsilon}{\epsilon}

\renewcommand{\hat}{\wh}

\DeclareMathOperator*{\E}{{\bf {E}}}

\def\leftarrowcirc{\hbox{$\leftarrow$}\kern-1.5pt\hbox{$\circ$}}
\def\rightarrowcirc{\hbox{$\circ$}\kern-1.5pt\hbox{$\rightarrow$}}
\def\circtailcirc{\hbox{$\circ$}\kern-1.5pt\hbox{$-$}\kern-1.5pt\hbox{$\circ$}}
\def\startailstar{\hbox{$*$}\kern-1.5pt\hbox{$-$}\kern-1.5pt\hbox{$*$}}
\def\startail{\hbox{$*$}\kern-1.5pt\hbox{$-$}}
\def\tailcirc{\hbox{$-$}\kern-1.5pt\hbox{$\circ$}}
\def\circtail{\hbox{$\circ$}\kern-1.5pt\hbox{$-$}}
\def\starrightarrow{\hbox{$*$}\kern-1.5pt\hbox{$\rightarrow$}}
\def\leftarrowstar{\hbox{$\leftarrow$}\kern-1.5pt\hbox{$*$}}

\newcommand{\Copt}{C^{\star}}
\newcommand{\MMM}{\mathcal{M}}
\newcommand{\DDD}{\mathcal{D}}
\newcommand{\SSS}{\mathcal{S}}
\newcommand{\PPP}{\mathcal{U}}
\newcommand{\MMMdom}{\mathcal{M}^{\text{dom}}}
\newcommand{\hatMMMdom}{\hat{\mathcal{M}}^{\text{dom}}}
\definecolor{mygreen}{RGB}{80,180,0}
\definecolor{b2}{RGB}{51,153,255}
\definecolor{mycy2}{RGB}{255,51,255}

\makeatletter
\newcommand*{\RN}[1]{\expandafter\@slowromancap\romannumeral #1@}
\makeatother
\usepackage{lineno}

\begin{document}

\author{
    Raghavendra Addanki\thanks{University of Massachusetts Amherst. \texttt{raddanki@cs.umass.edu}.}
  \and
  Shiva Prasad Kasiviswanathan\thanks{Amazon. \texttt{ kasivisw@gmail.com} }
  \and
}
\date{}
\title{Collaborative Causal Discovery with Atomic Interventions}
\maketitle

\begin{abstract}
We introduce a new Collaborative Causal Discovery problem, through which we model a common scenario in which we have multiple independent entities each with their own causal graph, and the goal is to simultaneously learn all these causal graphs. We study this problem without the causal sufficiency assumption, using Maximal Ancestral Graphs (MAG) to model the causal graphs, and assuming that we have the ability to actively perform independent single vertex (or atomic) interventions on the entities. If the $M$ underlying (unknown) causal graphs of the entities satisfy a natural notion of clustering, we give algorithms that leverage this property, and recovers all the causal graphs using roughly logarithmic in $M$ number of atomic interventions per entity. These are significantly fewer than $n$ atomic interventions per entity required to learn each causal graph separately, where $n$ is the number of observable nodes in the causal graph. We complement our results with a lower bound and discuss various extensions of our collaborative setting.
\end{abstract}

\section{Introduction}\label{sec:intro}
In this paper, we introduce a new model for {\em causal discovery}, the problem of learning all the causal relations between variables in a system.  Under certain assumptions, using just observational data, some ancestral relations as well as certain causal edges can be learned, however, many observationally equivalent structures cannot be distinguished~\citep{zhang2008causal}. Given this issue,  there has been a growing interest in
learning causal structures using the notion of an \emph{intervention} described in the Structural Causal Models (SCM) framework introduced by~\citet{pearl}. 

As interventions are expensive (require carefully controlled experiments) and performing multiple interventions is time-consuming, an important goal in causal discovery is to design algorithms that utilize simple (preferably, single variable) and fewer interventions~\cite{shanmugam2015learning}. However, when there are latents or unobserved variables in the system, in the worst-case, it is not possible to learn the exact {causal DAG} without intervening on every variable at least once. Furthermore, multivariable interventions are needed in presence of latents~\citep{addanki2020efficient}. 

\begin{figure}[!h]
    \centering
    \begin{minipage}{0.7\textwidth}
    \includegraphics[scale=0.8]{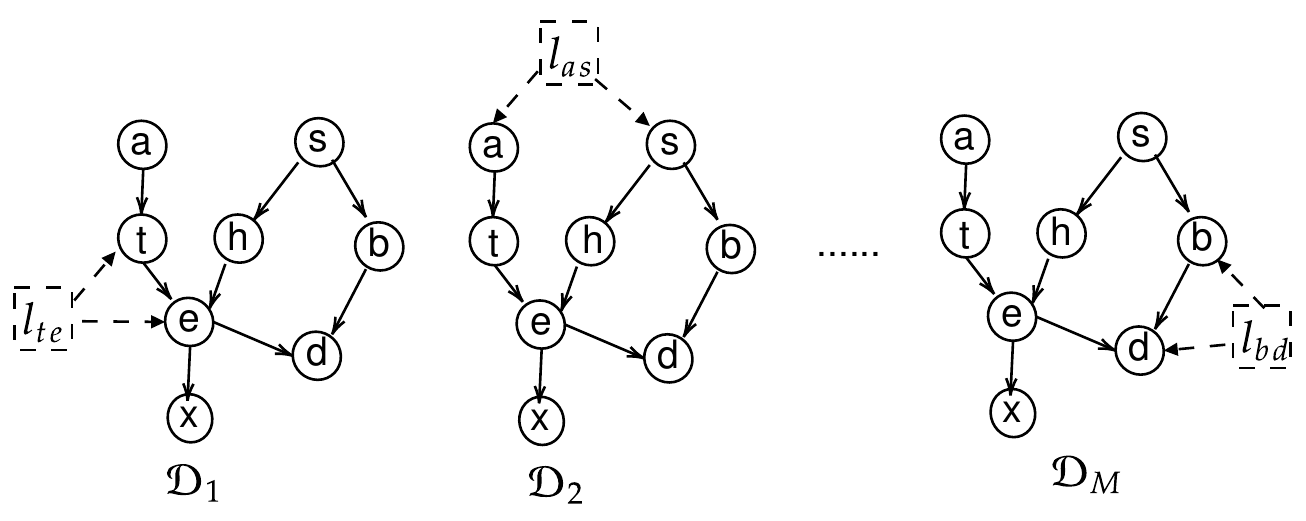}
    \end{minipage}
    \begin{minipage}{0.28\textwidth}
    \caption{Examples of $M$ causal graphs constructed from Lung Cancer dataset~\citep{lauritzen1988local}. Here, the causal graphs differ only in the presence of latents (nodes with dotted square box), but they could differ elsewhere too.}
    \label{fig:mags_intro}
    \end{minipage}
\end{figure}
On the other hand, in a variety of applications, there is no one true causal structure, different entities participating in the application might have different causal structures~\citep{gates2012group,ramsey2011meta,joffe2012causal}. For example, see figure ~\ref{fig:mags_intro}. In these scenarios, generating a single causal graph by pooling data from these different entities might lead to flawed conclusions~\citep{saeed2020causal}. Allowing for interventions, we propose a new model for tackling this problem, referred here as \emph{Collaborative Causal Discovery}, which in its simplest form states that: given a collection of  entities, each associated with an individual unknown causal graph and generating their own independent data samples, learn all the causal graphs while minimizing the number of interventions for every entity. An underlying assumption is that each entity on its own generates enough samples in both the observational and interventional settings so that  conditional independence tests can be carried out \emph{accurately} on each entity separately. To motivate this model of collaborative causal discovery, let us consider two different scenarios.

\begin{list}{{\bf (\alph{enumi})}}
	{\usecounter{enumi}
		\setlength{\leftmargin}{10pt}
		\setlength{\listparindent}{10pt}
		\setlength{\parsep}{-2pt}	
}
\item Consider a health organization interested in controlling incidence of a particular disease. The organization has a set of $M$ individuals (entities) whose data it monitors and can advise interventions on. Each individual is an independent entity that generates its own set of separate data samples\footnote{As is common in causal discovery, for the underlying conditional independence tests, the data is assumed to be i.i.d.\ samples from the interventional/observational distributions.}. In a realistic scenario, it is highly unlikely that all the $M$ individuals share the same causal graph (e.g., see Figures~\ref{fig:g1} and~\ref{fig:g2} from~\citet{joffe2012causal} in Appendix~\ref{app:model}). It would be beneficial for the organization to collaboratively learn all the causal graphs together. The challenge is, {\em a priori} the organization does not know the set of possible causal graphs or which individual is associated with which graph from this set.
\item An alternate setting is where, we have $M$ companies (entities)  wanting to work together to improve their production process. Each company generates their own data (e.g., from their machines) which they can observe and intervene on~\citep{nguyen2016fault}. Again if we take the $M$ causal graphs (one associated with each company) it is quite natural to expect some variation in their structure, more so because we do not assume {\em causal sufficiency} (i.e., we allow for latents). Since interventions  might need expensive and careful experimental organization, each company would like to reduce their share of interventions.
\end{list}
The collaborative aspect of learning can be utilized if we assume that there is some underlying (unknown) clustering/grouping of the causal graphs on the entities.

\subsection{Our Contributions}
We formally introduce the collaborative causal discovery problem in Section~\ref{sec:model}. We assume that we have a collection of $M$ entities that can be partitioned into $k$ clusters such that any pair of entities belonging to two different clusters are separated by large distance (see Definition~\ref{def:dist})  in the causal graphs. Due to presence of latents variables, we use a family of mixed graphs known as  {\em maximal ancestral
graphs} (MAGs) to model the graphs on observed variables. Each entity is associated with a MAG.

In this paper, we focus on designing algorithms that have {\em worst-case} guarantees on the number of atomic interventions needed to recover (or approximately recover) the MAG of each entity. We assume that there are $M$ MAGs one for each entity over the same set of $n$ nodes. Learning a MAG with atomic interventions, in worst case requires $n$ interventions (see Proposition~\ref{prop:lb_naive}). We show that this bound can be substantially reduced if the $M$ MAGs satisfy the property that every pair of MAGs from different clusters have {\em at least} $\alpha n$ nodes whose direct causal relationships are different. We further assume that entities belonging to same cluster have similar MAGs in that every pair of them have {\em at most} $\beta n$ ($\beta < \alpha$) nodes whose direct causal relationships are different. We refer to this clustering of entities as $(\alpha,\beta)$-clustering (Definition~\ref{def:newalphabeta}). A special but important case is when $\beta = 0$, in which case all the entities belonging to the same cluster have the same causal MAG (referred to as $\alpha$-clustering, Definition~\ref{def:alphaclustering}). An important point to notice is that while we assume there is a underlying clustering on the entities, it is {\em learnt} by our algorithms. Similar assumptions are  common for recovering the underlying clusters, in many areas, for e.g., crowd-sourcing applications~\citep{ashtiani2016clustering, awasthi2012center}. 

We first start with the observation that under  $(\alpha,\beta)$-clustering, even entities belonging to the same cluster could have a different MAG, which makes exact recovery hard without making a significant number of interventions per entity. We present an algorithm that using at most $O(\Delta \log(M/\delta)/(\alpha-\beta)^2)$  many interventions per entity, with probability at least $1-\delta$ (over only the randomness of the algorithm), can provably recover an {\em approximate} MAG for each entity. The approximation is such that for each entity we generate a MAG that is at most $\beta n$ node-distance from the true MAG of that entity (see Section~\ref{sec:alphabeta}). Here, $\Delta$ is the maximum undirected degree of the causal MAGs. Our idea here is to first recover the underlying clustering of entities by using a randomized set of interventions. Then, we distribute the interventions across the entities in each cluster, thereby,  ensuring that the number of interventions per entity is small. By carefully combining the results learnt from these interventions we construct the approximate MAGs.

Under the slightly more restrictive $\alpha$-clustering assumption, we present algorithms that can {\em exactly} recover all the MAGs using at most $\min \left\{ O(\Delta\log(M/\delta)/\alpha), O(\log(M/\delta)/\alpha + k^2) \right\} $ many interventions per entity (see Section~\ref{sec:graph_recovery_general}). Again, randomization plays an important role in our approach. 

Complementing these upper bounds, we give a lower bound using Yao's minimax principle~\citep{yao1977probabilistic} that shows for any (randomized or deterministic) algorithm $\Omega(1/\alpha)$ interventions per entity is required for this causal discovery problem. This implies the $1/\alpha$ dependence in our upper bound in the $\alpha$-clustering case is optimal. 

In Section~\ref{sec:expts}, we show experiments on data generated from both real and synthetic networks with added latents and demonstrate the efficacy of our algorithms for learning the underlying clustering and the MAGs.

\subsection{Related Work} 
A number of algorithms, working under various assumptions, for learning causal graph (or a causal DAG) using interventions have been proposed in the literature, e.g.,~\citep{hyttinen2013experiment, ghassami2018budgeted, shanmugam2015learning, neurips17, neurips18, addanki2020efficient,  addanki2021intervention, kocaoglu2019characterization, jaber2020causal}. {\citet{saeed2020causal}  consider a model where the observational data is from a mixture of causal DAGs, and outline ideas that recover a \emph{union graph} (up to Markov equivalence) of these DAGs, without any interventions. Our setting is not directly comparable to theirs, as we have entities generating data and doing conditional independence tests independently (no pooling of data from entities), but show stronger guarantees for recovering causal graphs, assuming atomic interventions.}

\section{Our Model and Problem Statement} \label{sec:model}

In this section, we introduce the collaborative causal discovery problem. We start with some notations.

\noindent\textbf{Notation.} Following the SCM framework \citep{pearl}, we represent the set of random variables of interest by $V \cup L$ where $V$ represents the set of endogenous (observed) variables that can be measured and $L$ represents the set of exogenous (latent) variables that cannot be measured. We do not deal with selection bias in this paper. Let $|V| = n$. 

We assume that the causal Markov condition and faithfulness holds for both the observational and interventional distributions~\citep{hauser2012characterization}.  We use Conditional Independence (CI) Tests as defined below.

\smallskip
\noindent \textbf{Conditional Independence (CI) Tests}.  Conditional independence tests are an important building block in causal discovery.
\begin{itemize}
\item[(i)] CI-test in observational distribution: Given $u,v \in V$, $Z \subset V$ check whether $u$ is independent of $v$ given $Z$, denoted by $u \indep v \mid Z$. 
\item[(ii)] CI-test in interventional distribution: Given $u,v \in V$, $Z \subset V$, and $w \in V$, check whether $u$ is independent of $v$ given $Z$ in the interventional distribution of $w$, denoted by $u \indep v \mid Z, \doo(w)$ where $\doo(w$) is the intervention on the variable~$w$.
\end{itemize}

Throughout this paper, unless otherwise specified, a path between two nodes is an undirected path. A path of only directed edges is called a directed path. $u$ is called an ancestor of $v$ and $v$ a descendant of $u$ if $ u = v$ or there is a directed path from $u$ to $v$. A directed cycle occurs in $G$ when $u \rightarrow v$ is in $G$ and $v$ is an ancestor of $u$.

\subsection{Maximal Ancestral Graphs}

\noindent\textbf{Background on MAGs.} Ancestral graphical models were introduced motivated by the need to represent data generating processes that involve latent variables. In this paper, we work with a class of graphical models, the maximal ancestral graph (MAG), which are a generalization of DAGs and are closed under marginalization and conditioning~\citep{richardson2002ancestral}. A maximal ancestral graph (MAG) is a (directed) mixed graph that may contain two kinds of edges: directed edges ($\rightarrow$) and bi-directed edges ($\leftrightarrow$). Before defining a MAG, we need some preliminaries.

Consider a mixed graph $\mathcal{G}$. Given a path $\pi = \langle u,\dots,w,\dots,v \rangle$, $w$ is a collider on $\pi$ if the two edges incident to $w$ in $\pi$ are both into $w$, that is, have an arrowhead into $w$; otherwise it is called a non-collider on $\pi$. Let $S$ be any subset of nodes in the graph $\mathcal{G}$. An {\em inducing path} relative to $S$ is a path on which every node not in $S$ (except for the endpoints) is a collider on the path and every collider
is an ancestor of an endpoint of the path.

\begin{definition} \label{def:mag} A mixed graph is called a maximal ancestral graph (MAG) if
\begin{enumerate}
    \item The mixed graph is {\em ancestral}, i.e., it has no directed cycles, and whenever
there is a bidirected edge $u \leftrightarrow v$, then there is no directed path from $u$ to $v$ or $v$ to $u$.
\item There is no inducing path between any two non-adjacent nodes.
\end{enumerate}
\end{definition}

It is straightforward to extend the notion of d-separation in DAGs to mixed graphs using the notion of m-separation ~\citep{richardson2002ancestral}.
\begin{definition} \label{def:msep}
In a mixed graph, a path $\pi$ between nodes $u$ and $v$ is 
m-connecting relative to a (possibly empty) set of nodes $Z$ with $u,v \notin Z$ if
\begin{enumerate}
    \item every non-collider on $\pi$ is not a member of $Z$;
    \item every collider on $\pi$ is an ancestor of some member of $Z$.
    \end{enumerate}
$u$ and $v$ are said to be m-separated by $Z$ if there is no m-connected path between $u$ and $v$ relative to
$Z$.
\end{definition}

\noindent\textbf{Conversion of a DAG to a MAG.} 
 The following construction gives us a MAG $\MMM$ from a DAG $\DDD$:
\begin{enumerate}
    \item for each pair of variables $u, v \in V$ , $u$ and $v$ are adjacent in $\MMM$ if and only if there is an
inducing path between them relative to $L$ in $\DDD$. The skeleton or the undirected graph constructed from PAG $\PPP$ (obtained using FCI~\citep{spirtes2000causation}) by ignoring the directions of edges captures all the edges in $\MMM$.

\item for each pair of adjacent variables $u,v$ in $\MMM$, orient the edge as $u \rightarrow v$ in $\MMM$ if $u$ is an
ancestor of $v$ in $\DDD$; orient it as $u \leftarrow v$ in $\MMM$ if $v$ is an ancestor of $u$ in $\DDD$; orient it as $ u \leftrightarrow v$ in $\MMM$ otherwise.
\end{enumerate}
Several DAGs can lead to the same MAG (See Figure~\ref{fig:3}).  Essentially a MAG represents a set of DAGs that have the exact same d-separation structures and ancestral relationships
among the observed variables. By construction, the MAG is unique for a given DAG.

As a further evidence to the claim that interventions are required, see Figure~\ref{fig:pag_distance} in Appendix~\ref{app:model}, that gives an example of two MAGs separated by a distance of $\frac{n}{2}$ (where $n$ is number of observable nodes) and have the same Partial Ancestral Graph identified by FCI~\citep{zhang2008completeness}.

\subsection{Our Model}
We assume that we have access to $M$ entities labeled $1,\dots,M$, each of which can independently generate their own observational and interventional data. {Each entity $i$ has an associated causal DAG $\DDD_i$ over $V \cup L_i$, where $L_i$ represents the latent variables of entity $i$}. In modeling the problem of learning $\DDD_i$, complications arise in at least two ways:

\begin{list}{{\bf (\roman{enumi})}}
	{\usecounter{enumi}
		\setlength{\leftmargin}{10pt}
		\setlength{\listparindent}{10pt}
		\setlength{\parsep}{-1pt}	
}

\item \textbf{Uniqueness}.  First, with just observational data, if $\DDD_1,\dots,\DDD_M$ are Markov equivalent, then without additional strong assumptions they cannot be distinguished even if they are all structurally different. To overcome the problem of being not identifiable within equivalence class, we allow for interventions on observed variables. Our objective will be to minimize these interventions. In particular, since each of these entities independently generate their own data, so we aim to reduce the number of interventions needed per entity. In causal discovery, minimizing the number of interventions while ensuring that they are of small size is an active research area~\citep{pearl1995causal, shanmugam2015learning, ghassami2018budgeted,ghassami2019interventional}.

\item \textbf{Latents}. Second complication comes due to presence of latents in the system. Even with power of single variable interventions, the structure of a causal DAG is not uniquely identifiable. (see, e.g., Figure~\ref{fig:diff}). Similarly, we can show that using single vertex interventions, we cannot exactly recover wider class of acyclic graphs like ADMGs (Acyclic Directed Mixed Graphs). 

\begin{figure}[!th]
\centering
	\begin{subfigure}[t]{0.2\textwidth}	\centering
        \begin{tikzpicture}[shorten >=1pt, auto, node distance=.25cm, thick,
            node/.style={circle,draw=black,scale=0.4, fill=white,font=\sffamily\Large\bfseries}, edge/.style={->,> = latex',draw=black, thick}]
            \node[node] (t) at (0.5,-0.75) {$\mathbf t$};
            \node[node] (x) at (1.5,0) {$\mathbf x$};
            \node[node] (y) at (3,0) {$\mathbf y$};
            \node[node] (z) at (3,1.1) {$\mathbf z$};
            \node[node, rectangle, dashed] (l) at (2, 1) {$\mathbf{l_{xy}}$};
            \node[node, rectangle, dashed] (l2) at (2, -0.75) {$\mathbf{l_{ty}}$};
            \draw[edge] (t) to (x);
            \draw[edge] (x) to (y);
            \draw[edge] (y) to (z);
            \draw[edge] (l) to (x);
            \draw[edge] (l) to (y);
            \draw[edge] (l2) to (t);
            \draw[edge] (l2) to (y);
            \end{tikzpicture}
        \caption{DAG $\DDD_1$}
        \label{fig:1}
\end{subfigure}
  \begin{subfigure}[t]{0.2\textwidth} 
   \centering
        \begin{tikzpicture}[shorten >=1pt, auto, node distance=.25cm, thick,
            node/.style={circle,draw=black,scale=0.4, fill=white,font=\sffamily\Large\bfseries},
             edge/.style={->,> = latex',draw=black, thick}]
            \node[node] (t) at (0.5,-0.75) {{$\mathbf t$}};
            \node[node] (x) at (1.5,0) {$\mathbf x$};
            \node[node] (y) at (3,0) {$\mathbf y$};
            \node[node] (z) at (3,1.1) {$\mathbf z$};
            \node[node, rectangle, dashed] (l) at (2, 1) {$\mathbf{l_{xy}}$};
            \draw[edge] (t) to (x);
            \draw[edge] (t) to (y);
            \draw[edge] (x) to (y);
            \draw[edge] (y) to (z);
            \draw[edge] (l) to (x);
            \draw[edge] (l) to (y);
            \end{tikzpicture}
        \caption{DAG $\DDD_2$}
        \label{fig:2}
         \end{subfigure}
  \begin{subfigure}[t]{0.2\textwidth} 
  \centering
        \begin{tikzpicture}[shorten >=1pt, auto, node distance=.25cm, thick,
            node/.style={circle,draw=black,scale=0.4, fill=white,font=\sffamily\Large\bfseries},
            edge/.style={->,> = latex',draw=black, thick}]
            \node[node] (t) at (0.5,-1.1) {$\mathbf t$};
            \node[node] (x) at (1,0) {$\mathbf x$};
            \node[node] (y) at (2.5,0) {$\mathbf y$};
            \node[node] (z) at (2,-1.1) {$\mathbf z$};
            \draw[edge] (t) to (x);
            \draw[edge] (x) to (y);
            \draw[edge] (t) to (y);
            \draw[edge] (y) to (z);
            \end{tikzpicture}
        \caption{MAG for $\DDD_1$ and $\DDD_2$}
                  \label{fig:3}
         \end{subfigure}
         \caption{Different DAGs with same MAG. It is easy to observe that, no single vertex interventions can differentiate $\DDD_1$ from~$\DDD_2$.}
                  \label{fig:diff}
\end{figure}
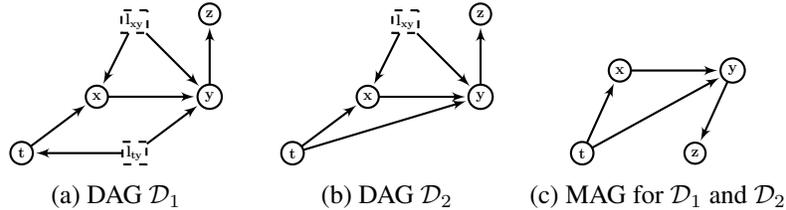

We overcome this problem, \textit{by focusing our attention on  maximal ancestral graphs (MAGs) of these entities.} MAGs are graphs on observed variables that capture the presence of latents through
bidirected edges~\citep{richardson2002ancestral}. 
A mixed graph containing directed ($\leftarrow$) and bidirected ($\leftrightarrow$) edges is said to be {\em ancestral} if it has no directed cycles, and whenever there is a bidirected edge $u \leftrightarrow v$, then there is no directed path from $u$ to $v$ or from $v$ to $u$.  An ancestral graph on $V$ (observables) is said to be {\em maximal}, if, for
every pair of nonadjacent vertices $u,v$, there exists a set $Z \subset V$ with $u,v \notin Z$ such
that $u$ and $v$ are $m$-separated (similar to $d$-separation, see Definition~\ref{def:msep}) conditioned on $Z$. {Every DAG with latents (and selection variables) can be transformed into a unique MAG over the observed variables}~\citep{richardson2002ancestral}.

\end{list}

Given the $M$ entities, let $\MMM_i$ denote the MAG associated with entity $i$ (the MAG constructed from the DAG $\DDD_i$). Our goal is to collaboratively learn all these MAGs $\MMM_1,\dots,\MMM_M$ while minimizing the maximum number of interventions per entity. 

To facilitate this learning, we make a natural underlying clustering assumption that partitions the entities based on their respective MAGs such that: (i) any two entities belonging to the same cluster have MAGs that are ``close'' to each other, (ii) any two entities belonging to different clusters have MAGs that are ``far'' apart. Before stating this assumption formally, we need some definitions.

For MAG $\MMM_i=(V,E_i)$, we denote the children (through outgoing edges), parent (through incoming edges), and spouse (through bidirected edges) of a node $u \in V$ as
\begin{equation} \label{eqn:outin}
\out{i}{u} = \{ v \mid u \rightarrow v \in E_i \}, \  \inc{i}{u} = \{ v \mid u \leftarrow v \in E_i\}, \ \bi{i}{u} = \{ v \mid u \leftrightarrow v \in E_i \}.
\end{equation} 

Also, define an incidence set for a vertex $u \in V$ which contains an entry $(v,\mbox{type})$ for every node $v$ adjacent to $u$ as
\begin{equation} \label{eqn:N}
   \NN{i}{u} = \left \{\begin{array}{ll}
        (v,\mbox{tail}) & \text{if} \ u \rightarrow v \in E_i\\
        (v,\mbox{head}) & \text{if} \  u \leftarrow v \in E_i\\
        (v,\mbox{bidirected}) & \text{if} \ u \leftrightarrow v \in E_i
        \end{array} \right\}.
\end{equation}
Note that $|\NN{i}{u}|$ is the undirected degree of $u$ in $\MMM_i$. We now define a distance measure between MAGs that captures structural similarity between them. 
\begin{definition} \label{def:dist}
Given two MAGs $\MMM_i = (V,E_i)$ and $\MMM_j = (V,E_j)$, define the node-difference as the set: $\diff(\MMM_i,\MMM_j) = \{ u \in V \mid \NN{i}{u} \neq \NN{j}{u} \}$,
and the node-distance as the cardinality of this set: 
$d(\MMM_i,\MMM_j) = |\diff(\MMM_i,\MMM_j)| =  |\{ u \in V \mid \NN{i}{u} \neq \NN{j}{u} \}|$.
\end{definition}
Intuitively, the node distance captures the number of nodes whose incidence relationships differ.  It is easy to observe that the node distance is a distance metric, and captures a strong structural similarity between the graphs. Two graphs $\MMM_i,\MMM_j$ are identical iff $d(\MMM_i,\MMM_j) = 0$.
{For e.g., in Figure~\ref{fig:cluster_MAGs}, we have two MAGs that satisfy $d(\MMM_{12}, \MMM_{13}) = 2$ as $\diff(\MMM_{12}, \MMM_{13}) = \{ x, z \}$, where  $d(\MMM_{12}, \MMM_{21}) = 3$ as $\diff(\MMM_{12}, \MMM_{21}) = \{ x, y, z \}$.}

We are now ready to define a simple clustering property on MAGs. 

\begin{figure}[!ht]
\centering
\begin{minipage}{0.60\textwidth}
\centering
    \includegraphics[scale=0.75]{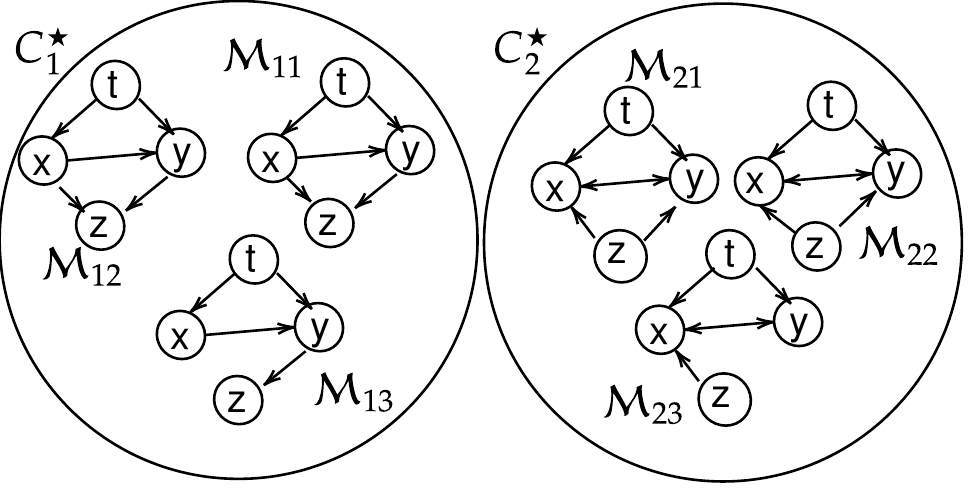}
\end{minipage}
\begin{minipage}{0.39\textwidth}
    \caption{{MAGs with $(\alpha=0.75, \beta=0.5)$-clustering. Every pair of graphs in $\Copt_1$ and $\Copt_2$ differ in at least $3 (= 0.75 \times 4)$ nodes, while pairs of graphs within clusters differ by at most $2 (= 0.5\times 4)$ nodes.}}
    \label{fig:cluster_MAGs}
\end{minipage}
\end{figure}

\begin{definition}[$(\alpha,\beta)$-clustering] \label{def:newalphabeta}
Let $\MMM_1,\dots,\MMM_M$ be a set of $M$ MAGs. We say that this set of MAGs satisfy the  $(\alpha,\beta)$-clustering property, with $\alpha > \beta \geq 0$, if there exists a partitioning of $[M]$ into sets (clusters) $\Copt_1,\dots,\Copt_k \subset [M]$ (for some $k \in \mathbb{N}$) such that for all $(i,j) \in [M] \times [M]$:

\begin{CompactEnumerate}
\item[(i)] if $i$ and $j$ belong to same set (cluster), then $d(\MMM_i,\MMM_j) \leq \beta n$;
\item[(ii)] if $i$ and $j$ belong to different sets (clusters), then $d(\MMM_i,\MMM_j) \geq \alpha n$.
\end{CompactEnumerate}
\end{definition}

Under this definition, all the $M$ MAGs could be different. See, e.g., Figure~\ref{fig:cluster_MAGs}. With right setting of $\alpha>\beta$ we can capture any set of possible $M$ MAGs.  Therefore, an algorithm such as FCI~\citep{spirtes2000causation}, that constructs PAGs might not be able to recover the clusters, as all the PAGs could be different, and the node-distance between PAGs does not correlate well with the node-distance between corresponding MAGs (e.g., see Figure~\ref{fig:pag_distance} in Appendix~\ref{app:model}). We use the PAGs generated by FCI as a starting point for all our algorithms that is then refined. 

\smallskip
\noindent With this discussion, we introduce our collaborative causal graph learning problem as follows:
	\begin{tcolorbox}[] 
	\textbf{Assumption:} MAGs $\MMM_1,\dots,\MMM_M$ (associated with entities $1,\dots,M$ respectively) satisfying the $(\alpha, \beta)$-clustering property \\
	\textbf{Access to each entity:} Through conditional independence (CI) tests on observational and interventional distributions. Each entity generates their own (independent) data samples. \\
	\textbf{Goal:} Learn $\MMM_1,\dots,\MMM_M$ while minimizing the max.\ number of interventions per entity.
	\end{tcolorbox}

\smallskip
\noindent An interesting case of the Definition~\ref{def:newalphabeta} is when $\beta=0$.
\begin{definition}[$\alpha$-clustering] \label{def:alphaclustering}
We say a set of MAGs $\MMM_1,\dots,\MMM_M$ satisfy the $\alpha$-clustering property, if and only if they satisfy $(\alpha, 0)$-clustering property.
\end{definition}

Note that $\alpha$-clustering is {a natural} property, wherein each cluster is associated with a single unique MAG, and all entities in the cluster have the same MAGs, and same conditional independences.

\section{Causal Discovery under $(\alpha,\beta)$-Clustering Property} \label{sec:alphabeta}

In this section, we present our main algorithm for collaboratively learning causal MAGs under the $(\alpha,\beta)$-clustering property. Missing details from this section are presented in Appendix~\ref{app:recovery_naive} and \ref{app:alphabeta}. 

\begin{definition}[Partial Ancestral Graph (PAG)]Let $[\MMM_i]$ denote the Markov equivalence class of the MAG $\MMM_i$ and represented by the Partial Ancestral Graph (or PAG) $\PPP_i = (V, \hat{E}_i)$. Edges $\hat{E}_i$ have three kinds of endpoints given by arrowheads $(\leftarrow)$, circles $(\circtail)$ and tails $(-)$.
\end{definition}

All our algorithms are randomized, and succeed with high probability over the randomness introduced by the algorithm. The idea behind all our algorithms is to first learn the true clusters $\Copt_1,\dots,\Copt_k$ using very few interventions. Once the true clusters are recovered, the idea is to distribute the interventions across the entities in each cluster and merge the results learned to recover the MAGs (Section~\ref{sec:MAGlearn}). For our algorithms, a lower bound for $\alpha$ and upper bound for $\beta$ is sufficient. In practice, a clustering of the PAGs (generated from FCI algorithm) can provide guidance about these bounds on $\alpha,\beta$, or if we have additional knowledge that $\alpha \in [1-\epsilon, 1]$ and $\beta \in [0, \epsilon]$ for some constant $\epsilon > 0$, then, we can use a binary search based guessing strategy that increases our intervention bounds by ${\log^2 (n\epsilon)}/{(1-2\epsilon)^2}$ factor. It is important to note that none of our algorithms require the knowledge of the number of underlying clusters $k$.

\begin{algorithm}[H]
\begin{small}
\label{alg:identify_nbr}
\begin{algorithmic}[1]
\STATE \textbf{Input:}  node $u \in V$,  PAG $\PPP_i$ of entity $i$
\STATE \textbf{Output:} $\out{i}{u}$
\STATE $\out{i}{u} = \{ v \mid u \rightarrow v \in \PPP_i \}$
\FOR{$v \in \D{i}(u)$ such that $u \circtailcirc v \ \text{ or } \ u \rightarrowcirc v \in \PPP_i$ }
\IF{$u \notindep v \mid \doo(u)$}
\STATE $\out{i}{u} \leftarrow \out{i}{u} \cup \{ v \}$
\ENDIF
\ENDFOR
\STATE Return $\out{i}{u}$
\end{algorithmic}
\caption{\IdentifyNbr$(\PPP_i,u)$}
\end{small}
\end{algorithm}

\begin{algorithm}[H]
\begin{small}
\label{alg:identify_bidir}
\begin{algorithmic}[1]
\STATE \textbf{Input:}  node $u \in V$, PAG $\PPP_i$ of entity $i$
\STATE \textbf{Output:} $\bi{i}{u}$
\STATE $\bi{i}{u} = \{v \mid u \leftrightarrow v \in \PPP_i \}$
\FOR{\small{$v \in \D{i}(u)$ such that $u \circtailcirc v$ or $u \leftarrowcirc v$ or $u \rightarrowcirc v \in \PPP_i$ }}
\IF{$u \indep v \mid \doo(u)$ and $u \indep v \mid \doo(v) $}
\STATE $\bi{i}{u} \leftarrow \bi{i}{u} \cup \{ v \}$
\ENDIF
\ENDFOR
\STATE Return $\bi{i}{u}$
\end{algorithmic}
\caption{\Identifybidir$(\PPP_i,u)$}
\end{small}
\end{algorithm}

\noindent\textbf{Helper Routines.} Let $\D{i}(u)$ denote all nodes that are adjacent to $u$ in the PAG $\PPP_i$, i.e., $\D{i}(u) = \{ v \mid (u,v) \in \hat{E}_i \}$. Given the PAG $\PPP_i$, Algorithm~\IdentifyNbr identifies all the outgoing neighbors of any node $u$ in $\MMM_i$. We look at edges of the form $u \circtailcirc v$ or $u \rightarrowcirc v$ in $\PPP_i$ incident on $u$, and identify if $u \rightarrow v$ using the CI-test $u \indep v \mid \doo(u)$. This is based on the observation that any node $v$ that is a descendant of $u$ (including $\out{i}{u}$) satisfies $u \notindep v \mid \doo(u)$.  Algorithm~\Identifybidir identifies all the bidirected edges incident on $u$.  If there is an edge of the form $u \circtailcirc v$ or $u \leftarrowcirc v$ or $u \rightarrowcirc v$ in the PAG, and $v \not\in \out{i}{u}$ and $u \not\in \out{i}{v}$, then it must be a bidirected edge.

Using these helper routines, we give an Algorithm~\RecoverG that recovers any MAG $\MMM_i$ using $n$ atomic interventions. The missing details are presented in Appendix~\ref{app:recovery_naive}.

\smallskip
\noindent \textbf{Algorithm~\RecoverG}. For every $u \in V$, first identify outgoing neighbors using Algorithm~\IdentifyNbr and then identify all the bidirected edges incident on $u$ using Algorithm~\Identifybidir.

\begin{lemma}\label{lem:naive_main}
Algorithm~\RecoverG recovers all edges of $\MMM_i$, for an entity $i \in [M]$ using $n$ atomic interventions.
\end{lemma}

\noindent Complementing this, we show that $n$ interventions are also required. 

\begin{proposition}\label{prop:lb_naive}
There exists a causal MAG $\MMM$ such that every adaptive or non-adaptive algorithm requires $\Omega(n)$ many {atomic} interventions to recover $\MMM$. 
\end{proposition}

\subsection{Recovering the Clusters} \label{sec:alphabetaalg}

From the $(\alpha,\beta)$-clustering definition, we know that a pair of entities belonging to the same cluster have higher structural similarity between their MAGs than a pair of entities across different clusters.

Let us start with a simplifying assumption that $\beta = 0$ (i.e., $\alpha$-clustering). So, all the MAGs are separated by a distance of at least $\alpha n$. We make the observation that to identify that two MAGs, say $\MMM_i$ and $\MMM_j$ belong to different clusters, it suffices to find a node $u$ from the node-difference set $\diff(\MMM_i,\MMM_j)$ and checking their neighbors using Algorithms~\IdentifyNbr and~\Identifybidir. We argue that (see Claim~\ref{cl:separation}, Appendix~\ref{app:nolatents}), with probability at least $1-\delta$, we can identify one such node $u \in \diff(\MMM_i,\MMM_j)$ by sampling $2\log (M/\delta)/\alpha$ nodes uniformly from $V$ as $|\diff(\MMM_i,\MMM_j)| = d(\MMM_i, \MMM_j) \geq \alpha n$.\!\footnote{For theoretical analysis, our intervention targets are randomly chosen, even with the knowledge available from PAGs, because in the worst-case the PAGs might contain no directed edges to help decide which nodes to intervene on. In practice, though if we already know edge orientations from PAG we do not have to relearn them, and a biased sampling based on edges uncertainties in PAGs might be a reasonable approach.} However, this approach will not succeed when $\beta \neq 0$ because now we have MAGs in the same cluster that are also separated by non-zero distance. 

\begin{algorithm}[!ht]
\caption{\abBoundedDegree}
\begin{small}
\label{alg:boundlatents}
\begin{algorithmic}[1]
\STATE \textbf{Input:} $\alpha > 0$, $\beta \geq 0$ ($< \alpha$), confidence parameter $\delta > 0$, PAGs $\UUU_1,\dots,\UUU_M$ of $M$ entities
\STATE \textbf{Output:} Partition of $[M]$ into clusters
\STATE Let $S$ denote a uniform sample of $\frac{4\log(M/\delta)}{(\alpha-\beta)^2}$ nodes from $V$ selected with replacement.
\FOR{every entity $i \in [M]$ and $u \in S$}
\STATE $\out{i}{u} \leftarrow \IdentifyNbr(\PPP_i, u)$  \STATE $\bi{i}{u} \leftarrow \Identifybidir(\PPP_i, u)$
\STATE $\inc{i}{u} \leftarrow \D{i}(u) \setminus \left( \out{i}{u} \cup \bi{i}{u} \right)$ 
\STATE Construct $\NN{i}{u}$ (defined in~\eqref{eqn:N})
\ENDFOR
\STATE Let $\mathcal{P}$ denote an empty graph on set of entities $[M]$
\FOR{every pair of entities $i, j$}
\STATE Let $\Count(i, j) = \sum_{u \in S} \mathbf{1}\{ \NN{i}{u} = \NN{j}{u} \}$
\IF{$\Count(i, j) \geq \left( 1- \frac{\alpha+\beta}{2}\right)|S|$}
\STATE Include an edge between $i$ and $j$ in $\mathcal{P}$
\ENDIF
\ENDFOR
\STATE Return connected components in $\mathcal{P}$
\end{algorithmic}
\end{small}
\end{algorithm}

\smallskip
\noindent \textbf{Overview of Algorithm \abBoundedDegree.} We now build upon the above idea, to recover the true clusters $\Copt_1,\dots,\Copt_k$ when $\beta \neq 0$. As identifying a node $u \in \diff(\MMM_i,\MMM_j)$ is not sufficient, we maintain a count of the number of nodes among the sampled set of nodes $S$ that the pair of entities $i, j$ have the same neighbors, i.e., $\Count(i, j) = \sum_{u \in S} \mathbf{1}\{ \NN{i}{u} = \NN{j}{u} \}.$ {Based on a carefully chosen threshold value for the $\Count(i, j)$, that arises through the analysis of our randomized algorithm, we classify whether a pair of entities belong to the same cluster correctly.}

Overall, the idea here is to construct a graph $\mathcal{P}$ on entities (i.e., the node set of $\mathcal{P}$ is $[M]$). We include an edge between two entities $i$ and $j$ if $\Count(i, j)$ is above the threshold $(1-(\alpha+\beta)/2)|S|$. Using Lemma~\ref{lem:latent_alpha_beta}, we show that this threshold corresponds to the case where if the entities are from same true clusters, then the $\Count$ value corresponding to the pair is higher than the threshold; and if they are from different clusters it will be smaller, with high probability.  This ensures that every entity is connected only to the entities belonging to the same true cluster. We return the connected components in $\mathcal P$ as our clusters.

\smallskip
\noindent \textbf{Theoretical Guarantees.} In Algorithm~\abBoundedDegree, we construct a uniform sample $S$ of size $O(\log(M/\delta)/(\alpha-\beta)^2)$, and identify all the neighbors of $S$ for every entity $i \in [M]$. As we use ~\Identifybidir to identify all the bi-directed edges, the total number of interventions used by an entity for this step is at most $\Delta \cdot |S|$. Combining all the above, using the next lemma, we show that with high probability Algorithm~\abBoundedDegree recovers all the true clusters.

\begin{lemma}
\label{lem:latent_alpha_beta}
If the underlying MAGs $\MMM_1,\dots,\MMM_M$ satisfy $(\alpha,\beta)$-clustering property with true clusters $\Copt_1,\dots,\Copt_k$ and have maximum undirected degree $\Delta$. Then, the Algorithm \abBoundedDegree recovers the clusters $\Copt_1,\dots,\Copt_k$ with  probability at least $1-\delta$. Every entity $i \in [M]$ uses at most $4(\Delta+1) \log(M/\delta)/(\alpha-\beta)^2$ many atomic interventions.
\end{lemma}

\subsection{Learning Causal Graphs from $(\alpha,\beta)$-Clustering } \label{sec:MAGlearn}

In this section, we outline an approach to recover a close approximation of the causal MAGs of all the entities, after correctly recovering the clusters using Algorithm~\abBoundedDegree. First, we note that since the $(\alpha,\beta)$-clustering allows the MAGs even in the same cluster to be different, the problem of exactly learning all the MAGs is challenging (with a small set of interventions) as causal edges learnt for an entity may not be relevant for another entity in the same cluster. 

In the scenarios mentioned in the introduction, we expect the clusters to be more homogeneous, with many entities in the same cluster sharing the same MAG. We provide an overview of Algorithm~\Recovery that recovers one such MAG called \emph{dominant MAG} for every cluster. Consider a recovered cluster $\Copt_a$, and a partitioning $S^1_a, S^2_a, \cdots$ of MAGs such that all MAGs in a partition $S^i_a$ are equal for all $i$. We call the MAG $\MMMdom_a$ corresponding to the largest partition $S^{\text{dom}}_a$ as the \emph{dominant MAG} of $\Copt_a$. The dominant MAG of a cluster is parameterized by $\gamma_a = |S^{\text{dom}}_a|/|\Copt_a|$ (fraction of the MAGs in the cluster that belong to the largest partition). 

\smallskip
\noindent \textbf{Overview of Algorithm~\Recovery.} Consider a cluster $\Copt_a$. We recover the dominant MAG of this cluster, $\MMMdom_a$, by recovering all the neighbors of every node and carefully merging them. Our idea is to assign a node, selected uniformly at random, to every entity in $\Copt_a$, and recover the neighborhood  of the node using Algorithms~\IdentifyNbr and~\Identifybidir. If the clusters are large such that $|\Copt_a| \gg n$ (see Theorem~\ref{thm:dominantMAG} for a precise bound), we can show a large number of entities $T_u$ are assigned node $u$, and many of them will share the dominant MAG. We maintain a count $\NCount(i, u)$ of the number of times the entity $i$ agrees with other entities in $T_u$ about neighbors of $u$, and guarantee (with high probability) that the entity with the highest count will be that of dominant MAG. After merging the neighbors recovered for every node, we assign the resulting graph to every entity in the cluster.

\begin{algorithm}[H]
\caption{\Recovery}
\begin{small}
\label{alg:graphs_boundlatents}
\begin{algorithmic}[1]
\STATE \textbf{Input:} $\alpha > 0$, $\beta \geq 0$ ($< \alpha$), confidence parameter $\delta > 0$, PAGs $\UUU_1,\dots,\UUU_M$ of $M$ entities
\STATE \textbf{Output:} $ \hat{\MMM}_1, \hat{\MMM}_2, \cdots \hat{\MMM}_M$ representing set of $M$ MAGs.
\STATE Obtain clusters $\Copt_1, \Copt_2, \cdots, \Copt_k$ using Algorithm~\ref{alg:boundlatents}.
\FOR{every cluster $\Copt_a$ where $a \in [k]$}
\STATE Let $\hatMMMdom_a$ be an empty graph on the set of nodes $V$.
\STATE For every entity $i \in \Copt_a$, select a node $u \in V$ uniformly at random and assign it to $u$ represented by the set $T_u$. 
\FOR{every node $u \in V$ }
\FOR{every entity $i \in T_u$}
\STATE $\out{i}{u} \leftarrow \IdentifyNbr(\PPP_i, u)$  \STATE $\bi{i}{u} \leftarrow \Identifybidir(\PPP_i, u)$.
\STATE $\inc{i}{u} \leftarrow \D{i}(u) \setminus \left( \out{i}{u} \cup \bi{i}{u} \right)$. 
\STATE Construct $\NN{i}{u}$ (defined in~\eqref{eqn:N}) and calculate $\NCount(i, u) = \sum_{ j \in T_u : j\neq i} \mathbf{1} \{ N_i(u) = N_j (u) \}$ 
\ENDFOR
\STATE Let $u_{\max} \leftarrow \arg \max_{i \in T_u} \NCount(i, u)$.
\STATE Set neighbors of $u$ in $\hatMMMdom_a$ to the set $N_{u_{\max}}(u)$.
\ENDFOR
\STATE For every entity $i \in \Copt_a$, set $\hat{\MMM}_i = \hatMMMdom_a$.
\ENDFOR
\STATE Return $ \hat{\MMM}_1, \hat{\MMM}_2, \cdots \hat{\MMM}_M$
\end{algorithmic}
\end{small}
\end{algorithm}

As the entities satisfy $(\alpha, \beta)$-clustering property, for all entities the recovered MAGs (dominant MAGs) are close to the true MAGs, and within a distance of at most $\beta n$. Note that any MAG from the cluster is within a distance of at most $\beta n$ due to $(\alpha, \beta)$-clustering property, but naively generating a valid MAG from a cluster will require $n$ interventions on one entity from Proposition~\ref{prop:lb_naive}. Our actual guarantee is somewhat stronger, as in fact, for the entities whose MAGs are dominant in their cluster, we do recover the exact MAGs. We have the result:

\begin{theorem}\label{thm:dominantMAG}
Suppose $\MMM_1, \MMM_2, \cdots \MMM_M$ satisfy $(\alpha, \beta)$ clustering property. If $\Copt_a = \Omega(n\log(n/M\delta)(2\gamma_a -1)^2)$ for all $a \in [k]$, then, Algorithm~\Recovery recovers graphs $\hat{\MMM}_1, \cdots \hat{\MMM}_M$ such that for every entity $i \in [M]$, we have $d(\MMM_i, \hat{\MMM}_i) \leq \beta n$  with probability $1-\delta$. Every entity uses at most $ (\Delta+1) + 4(\Delta+1) \log(M/\delta)/(\alpha-\beta)^2$ many atomic interventions.
\end{theorem}
\section{Causal Discovery under $\alpha$-Clustering Property}\label{sec:graph_recovery_general}

In the previous section, we discussed the more general $(\alpha,\beta)$-clustering scenario where we manage to construct a good approximation to all the MAGs. Now, we show that we can in fact recover all the MAGs exactly, if we make a stronger assumption. Missing details from this section are presented in Appendix~\ref{app:graph_recovery}.

Suppose the MAGs $\MMM_1,\dots,\MMM_M$ of the $M$ entities satisfy the $\alpha$-clustering property (Defn.~\ref{def:alphaclustering}). Firstly, we can design an algorithm similar to Algorithm~\abBoundedDegree (see Algorithm~\BoundedDegree, Appendix~\ref{app:boundedlatents}) that recovers the causal MAGs exactly with $O(\Delta \log(M/\delta)/\alpha)$ many interventions per entity, succeeding with probability $1-\delta$. Note that this has a better $1/\alpha$ term in the intervention bound, instead of $1/\alpha^2$ (when $\beta = 0$) term arising in Theorem~\ref{thm:dominantMAG}. In absence of latents, we can further improve it to $O(\log(M/\delta)/\alpha)$ many interventions per entity (see Algorithm~\NoLatents, Appendix~\ref{app:nolatents}).

In this section, we present another approach (Algorithm~\GenRecovery) with an improved result that requires fewer number of interventions, even when $\Delta$ is big, provided that each cluster has at least $\Omega(n \log (M/\delta))$ entities. Missing details of Algorithm~\GenRecovery are in Appendix~\ref{app:graph_recovery_general}.

\smallskip
\noindent \textbf{Overview of Algorithm~\GenRecovery}. First, using a similar approach as Algorithm~\abBoundedDegree, we construct a uniform sample $S \subseteq V$, and find all the outgoing neighbors of nodes in $S$, for every entity $i \in [M]$. Then, we construct a graph on entities denoted by $\mathcal{P}$, where we include an edge between a pair of entities if the outgoing neighbors of the set of sampled nodes $S$, and the set of neighbors in PAGs associated with the entities (obtained from FCI) are the same. However, due to the presence of bidirected edges, it is possible that the connected components of $\mathcal P$ may not represent the true clusters $\Copt_1,\dots,\Copt_k$. 

We make the observation that a pair of entities $i, j$ that have an edge in this $\mathcal{P}$ and from different true clusters, can differ only if there is a node $u$ such that $u$ has a bidirected edge $u \leftrightarrow v$ in $\MMM_i$, and a directed edge $u \leftarrow v$ in $\MMM_j$ (or vice-versa). Intervening on both $u$ and $v$ will separate these entities, our main idea is to ensure that this happens. First, we show how to \emph{detect}  if there are at least two true clusters in any connected component of $\mathcal{P}$. Then, we identify all the entities belonging to these two clusters and remove the edges between these entities in $\mathcal{P}$ and continue. 

More formally, let $T_1,\dots,T_{k'}$ be the partition of $[M]$ provided by the $k'$ connected components of $\mathcal{P}$ and some of these can contain more than one true cluster, hence $k' \leq k$ and we focus on detecting such events. Let $\pi : [M] \rightarrow V$ denote a mapping from the set of entities  to the nodes in $V$ such that $\pi(i)$ is chosen uniformly at random from $V$ for every entity $i$.  For every entity $i$, we intervene on the node $\pi(i)$. To detect that there are at least two clusters in a given subset $T_a$ of entities, we show that there are two entities $i, j$ with an edge in $\mathcal{P}$ and for some node $u \in S$, we can identify the neighbor $v \in \D{i}(u) \cap \D{j}(u)$ such that $u \leftrightarrow v$ is an edge in $\MMM_i$ and $u \leftarrow v$ is an edge in $\MMM_j$ (or vice-versa). As there are at least $\Omega(n \log (M/\delta))$ entities in each of these two true clusters in $T_a$, for some $i, j \in T_a$, we can show that  $\pi(i) = \pi(j) = v$ with probability at least $1-\delta$.

After detecting the event that a component $T_a$ of $\mathcal{P}$ contains entities from at least two different true clusters (say, $\Copt_b$ and $\Copt_c$) due to an edge $(u,v)$ as above, we intervene on $v$ for every entity in $T_a$. By intervening on $v$ (and $u \in S$), we can separate all entities in $T_a$ that belong to true clusters $\Copt_b$ and $\Copt_c$, and remove edges between such entity pairs from $\mathcal{P}$. 

We repeat this above procedure of refining $\mathcal{P}$. In each iteration, we will have removed all edges between every pair of entities belonging to at least two different true clusters. Since there are at most $k^2$ different true cluster pairs, after $k^2$ iterations the connected components remaining correspond to the true clusters (with high probability). This can be done without knowing the value of $k$, by checking whether the connected components in $\mathcal{P}$ change or not after each iteration of the above idea.

\begin{algorithm}[!ht]
\begin{small}
\caption{\GenRecovery}
\label{alg:latents}
\begin{algorithmic}
\STATE \textbf{Input:} $\alpha > 0$, confidence parameter $\delta > 0$, PAGs $\UUU_1,\dots,\UUU_M$ of $M$ entities $\UUU_1,\dots,\UUU_M$)
\STATE \textbf{Output:} Partition of $[M]$ into clusters
\STATE Let $S$ denote a uniform sample of $\frac{2\log(2M/\delta)}{\alpha}$ nodes from $V$ selected with replacement.
\FOR{every entity $i \in [M]$ and $u \in S$}
\STATE $\out{i}{u} \leftarrow \IdentifyNbr(i, u)$
\ENDFOR
\STATE Let $\mathcal{P}$ denote an empty graph on the set of entities $[M]$.
\FOR{every pair of entities $i, j$}
\IF{$\out{i}{u} = \out{j}{u}$ and $\D{i}(u) = \D{j}(u)$  $\ \forall u \in S$}
\STATE Include an edge between $i$ and $j$ in $\mathcal{P}$
\ENDIF
\ENDFOR
\STATE $\itr \leftarrow 1$, $\mathcal{P}_0 \leftarrow \mathcal{P}$
\WHILE{\textsc{True}}
\STATE $\mathcal{P}_{\itr} \leftarrow \mathcal{P}_{\itr-1}$
\STATE Let $T_{1},T_{2}, \cdots$ denote the components in $\mathcal P_{\itr}$.
\STATE For all $i \in [M]$, obtain interventional distribution on $\pi(i)$ picked u.a.r from $V$.
\IF{$\exists \mbox{ edge }  (i, j) \in \mathcal{P}_{\itr}$ in component $T_a$ such that $\pi(i) = \pi(j)$}
\STATE Let $v = \pi(i) = \pi(j)$
\IF{$v \in \bi{i}{u}, v \not\in \bi{j}{u}$ (or vice-versa) for some $u \in S$}
\STATE Intervene on $v$ for every entity in $T_a$.
\STATE Remove edge $(i', j')$ from $\mathcal{P}_{\itr}$ if $v \in \bi{i'}{u}$, $v \not\in \bi{j'}{u}$ (or vice-versa) for every $i', j' \in T_a$
\ENDIF
\ENDIF
\IF{the set of edges in $\mathcal{P}_{\itr}$ are \emph{same} as the set of edges in $\mathcal{P}_{\itr-1}$}
\STATE Return connected components in $\mathcal P_{\itr}$
\ENDIF
\STATE $\itr \leftarrow \itr + 1$
\ENDWHILE
\end{algorithmic}
\end{small}
\end{algorithm}

\subsection{From $\alpha$-Clustering to Learning Causal Graphs} \label{app:meta_main}


 Suppose that the underlying MAGs $\MMM_1, \MMM_2, \cdots, \MMM_M$ satisfy the $\alpha$-clustering property, our algorithms are based on first accurately recovering these clusters. The idea of going from clusters to MAGs is simple and is based on distributing the interventions across the entities in the cluster. We now discuss a meta-algorithm that returns the associated causal MAG of every entity given the true clustering. Our meta-algorithm takes as input the true clusters $\Copt_1, \Copt_2, \cdots, \Copt_k$ and recovers the MAGs associated with each of them. In any cluster $\Copt_b$ such that $|\Copt_b| < n$, our meta-algorithm uses an additional $\lceil n/|\Copt_b|\rceil$ many interventions for each entity in $\Copt_b$. For clusters satisfying $|\Copt_b| \geq n$, it uses an extra intervention per entity.

\smallskip
\noindent \textbf{Meta-Algorithm}. Consider a true cluster $\Copt_b$ ($b \in [k]$). {Construct a mapping $\phi$ that partitions the $n$ nodes in $V$ among all the entities in $\Copt_b$, such that no entity is assigned to more than $\lceil n/|\Copt_b| \rceil$ many nodes.} By definition, all entities in $\Copt_b$ have the same PAG. Let $\PPP$ be the common PAG. Construct a MAG $\MMM$ from $\PPP$ as follows. Consider an edge $(u,v)$ in $\UUU$. Let $u = \phi(i)$ and $v = \phi(j)$ where the entities $i,j \in \Copt_b$ are such that we intervene on node $u$ in entity $i$ and node $v$ in entity $j$ ({$i$ could be equal to $j$}). Now, if $v \in \out{i}{u}$, we add $u \rightarrow v$ into the graph $\MMM$, else if $u \in \out{j}{v}$, we add $u \leftarrow v$, and $u \leftrightarrow v$ otherwise. We assign graph $\MMM$ for every entity in $\Copt_b$. Repeating this procedure for every $\Copt_b$ generates the $M$ MAGs, one for each entity.

If the clusters are of size at least $n$, i.e., $\min_{b \in [k]} |\Copt_b| \geq n$, then, we have the following corollary. For additional details, see Appendix~\ref{app:meta}. 
\begin{corollary}\label{cor:cluster_to_graph}
Suppose there is an Algorithm $\mathcal{A}$ that recovers the true clusters $\Copt_1, \Copt_2, \cdots, \Copt_k$ of the underlying MAGs $\MMM_1, \MMM_2, \cdots, \MMM_M$ satisfying the $\alpha$-clustering property such that every entity $i \in [M]$ uses at most $f(M)$ interventions. Suppose $\min_{b \in [k]} |\Copt_b| \geq n$. Then, there is an algorithm that can learn all the MAGs $\MMM_1, \MMM_2, \cdots, \MMM_M$ such that every entity $i \in [M]$ uses at most $f(M) + 1$  many interventions.
\end{corollary}

Combining the best of the guarantees obtained using Algorithms~\GenRecovery and~\BoundedDegree, and the Corollary~\ref{cor:cluster_to_graph} we have:
\begin{theorem}\label{thm:latents}
If MAGs $\MMM_1,\dots,\MMM_M$ satisfy $\alpha$-clustering property with true clusters $\Copt_1,\dots,\Copt_k$ such that $\min_{b \in [k]} |\Copt_b| = \Omega(n \log (M/\delta))$. Then, there is an algorithm that exactly learns all these MAGs with probability at least $1-\delta$. Every entity $i \in [M]$ uses  
$\min \left\{ O(\Delta\log(M/\delta)/\alpha), O(\log(M/\delta)/\alpha + k^2) \right\}$ many atomic interventions. 
\end{theorem}

\subsection{Lower Bound on the Number of Interventions}\label{sec:lb}

We now give a lower bound on the number of atomic interventions needed for every algorithm that recovers the true clusters on the MAGs $\MMM_1, \MMM_2, \cdots \MMM_M$. Since a lower bound under $\alpha$-clustering is also a lower bound under $(\alpha,\beta)$-clustering, we work with the $\alpha$-clustering property here. First, we show that to identify whether a given pair of entities $i,j$ belong to the same  true cluster or not, 
every (randomized or deterministic) algorithm must make $\Omega(1/\alpha)$ interventions for both $i$ and $j$. Missing details from this section are presented in Appendix~\ref{app:lb}.

Our main idea here is to use the famous Yao's minimax theorem~\citep{yao1977probabilistic} to get lower bounds on randomized algorithms. Yao's theorem states that an {\em average case} lower bound on a deterministic algorithm implies a {\em worst case} lower
bound on randomized algorithms. 

To show a lower bound using Yao's minimax theorem, we construct a distribution $\mu$ on MAG pairs and show that every deterministic algorithm requires $\Omega(1/\alpha)$ interventions for distinguishing a pair of MAGs drawn from $\mu$. 

\smallskip
\noindent \textbf{Outline of the Lower Bound.} Our distribution $\mu$ places a probability of $1/2$ for pairs of MAGs that have distance zero and a probability of $1/2$ equally distributed among all pairs of MAGs with distance equal to $\alpha n$. This ensures that both the events considered are equally likely, and we show that to distinguish them, with success probability at least $2/3$ (over the distribution $\mu$), every deterministic algorithm must make $\Omega(1/\alpha)$ interventions for both the MAGs. Then, we use Yao's theorem to translate this into a worst case lower bound for any randomized algorithm. In particular, this means that any  algorithm that is based on recovering the clusters to construct the MAGs will require $\Omega(1/\alpha)$ interventions for every entity in $[M]$. 

\smallskip
\noindent \textbf{Details.} For the lower bound, consider the case when  $M=2$, and assuming causal sufficiency, where we wish to identify the clusters of two MAGs $\MMM_1, \MMM_2$. We observe that a lower bound on the number of interventions required for every entity in the case of identifying two clusters will also extend for the general case of identifying $k$ clusters with latents. 

Consider two MAGs $\MMM_1, \MMM_2$ on a node set $V$, with the promise that either $d(\MMM_1, \MMM_2)=0$  or  $d(\MMM_1, \MMM_2)= \alpha n$, and the goal is to identify which case holds. Note that in the first case the two entities are in the same cluster ($k=1$), and in the second case they are in different clusters ($k=2$).

Let $V = \{v_1,\dots,v_n\}$ be the set of observable nodes of these MAGs. Consider the node difference set of the MAGs $\MMM_1, \MMM_2$ given by $\diff(\MMM_1, \MMM_2)$ and let $e \in \{0,1\}^n$ denote its characteristic vector where $l$th coordinate of $e$ is $1$ iff $v_l \in \diff(\MMM_1, \MMM_2)$. We can observe that, under the above promise, $e$ is either $0^n$ or has exactly $\alpha n$ ones. Therefore, we have reduced our problem to that of finding whether the vector $e$ contains all zeros or not. Using this reduction, we focus on establishing a lower bound for this modified problem.

We want to check if a given $n$-dimensional binary vector is a zero vector, i.e., $0^n$ or not, with a promise that if it is not a zero vector, then, it contains $\alpha n$ coordinates with $1$ in them. Using Lemma~\ref{lem:veclb}, we show that $\Omega\left(\frac{1}{\alpha}\right)$ queries to co-ordinates of $x$ are required, for any randomized or deterministic algorithm to distinguish between these two cases.

\begin{lemma}\label{lem:veclb}
Suppose we are given a vector $x \in \{0, 1\}^n$ with the promise that either $x = 0^n$ or $x$ contains $\alpha n$ ones. In order to distinguish these two cases with probability more than $2/3$, every randomized or deterministic algorithm must make at least $\Omega(1/\alpha)$ queries to the coordinates of the vector $x$.
\end{lemma}

For the above problem of identifying whether a vector is zero or not, we can replace each coordinate query by an intervention on the corresponding node for the two entities (due to the equivalency between the two as explained above). Therefore, from Lemma~\ref{lem:veclb}, we have the following corollary about recovering the clusters. 
\begin{corollary}\label{cor:lb}
Suppose we are given two MAGs $\MMM_1$ and $\MMM_2$ corresponding to two entities, with the promise that either $d(\MMM_1, \MMM_2) = 0$ or $d(\MMM_1, \MMM_2) = \alpha n$. In order to distinguish these two cases with probability at least $2/3$, every  (randomized or deterministic) algorithm must make at least $\Omega(1/\alpha)$ interventions on both the entities.
\end{corollary}

\noindent Using Corollary~\ref{cor:lb}, we obtain the following final result about recovering clusters under $\alpha$-clustering property:
\begin{theorem}
Suppose the underlying MAGs $\MMM_1,\dots,\MMM_M$ satisfy $\alpha$-clustering property. In order to recover the clusters with probability $2/3$, every (randomized or deterministic) algorithm requires $\Omega(1/\alpha)$ interventions for every entity in $[M]$.
\end{theorem}
\begin{proof}
From Corollary~\ref{cor:lb}, we have that to identify whether two MAGs belong to the same cluster or not, we have to make at least $\Omega(1/\alpha)$ interventions for every entity. Therefore, to recover all the clusters, we have to make at least $\Omega(1/\alpha)$ many interventions for every entity $i \in [M]$.
\end{proof}
%

\section{Experimental Evaluation} \label{sec:expts}

In this section, we provide an evaluation of our approaches on data generated from real and synthetic causal networks for learning MAGs satisfying $(\alpha, \beta)$-clustering property. We defer additional details, results, and evaluation for $\alpha$-clustering to Appendix~\ref{app:expts}. 

\smallskip
\noindent\textbf{Causal Networks.} We consider the following real-world Bayesian networks from the \href{https://www.bnlearn.com/bnrepository/}{\em Bayesian Network Repository} which cover a wide variety of domains: \textit{Asia} (Lung cancer) (8 nodes, 8 edges), \textit{Earthquake} (5 nodes, 4 edges), \textit{Sachs} (Protein networks) (11 nodes, 17 edges), and \textit{Survey} (6 nodes, 6 edges). For the synthetic data, we use Erd\"os-R\'enyi random graphs (10 nodes). {We use the term ``causal network'' to refer to these ground-truth Bayesian networks.}

\smallskip
\noindent\textbf{Setup}. We used a personal Apple Macbook Pro laptop with 16GB RAM and Intel i5 processor for conducting all our experiments. We use the FCI algorithm implemented in~\citep{pcalg}. For every causal network, each experiment took less than 10 minutes to finish all the 10 runs. 

\smallskip
\noindent\textbf{(Synthetic) Data Generation}. We use following process for each of the five considered causal network (\textit{Asia}, \textit{Earthquake}, \textit{Sachs}, \textit{Survey}, and \textit{Erd\H{o}s-Renyi}). We construct causal MAGs for $M$ entities distributed among the clusters $\Copt_1, \Copt_2, \cdots \Copt_k$ equally, i.e., $|\Copt_i| = M/k$ for all $i \in [k]$. In our experiments, we set $k=2$ (i.e., two clusters), and start with $k=2$ DAGs that are sufficiently far apart. To do so, we create two copies of the original causal network $\DDD$, and denote the DAG copies by $\DDD_1$ and $\DDD_2$. For each of the DAGs $\DDD_1$ and $\DDD_2$, we select a certain number of pairs of nodes randomly, and include a latent variable between them, that has a causal edge to both the nodes. In our experiments, we used 2 latents per DAG. This results in two new DAGs $\DDD'_1$ and $\DDD'_2$. To ensure $\alpha n$ node distance between clusters, we modify $\DDD'_2$ using random changes until the two MAGs corresponding to the DAGs $\DDD'_1$ and $\DDD'_2$ are separated by a distance of $\alpha n$. These two MAGs, denoted by $\MMMdom_1$ and $\MMMdom_2$ form the dominant MAG for each of the two clusters. 

Then, we create $(1-\gamma) M/k = (1-\gamma) M/2$ copies of the dominant MAG and assign it to distinct entities in each cluster. Consider cluster $\Copt_1$ with dominant MAG $\MMMdom_1$, and corresponding DAG $\DDD'_1$. Note that each cluster has $M/k = M/2$ entities. For the remaining entities in $\Copt_1$, we start with $\DDD_1$ and include 2 latent variables between randomly selected pairs of nodes. Then, we repeat the previous procedure, of performing a series of random insertions or deletions of edges to the DAG until the distance between the corresponding MAG and $\MMMdom_1$ increases to $\beta n$. We follow the same procedure for cluster $\Copt_2$ with dominant MAG $\MMMdom_2$. Note that in this construction different entities could differ both in latents and their observable graphs. This construction ensures the entities satisfy $(\alpha,\beta)$-clustering property. As an example, see Figure~\ref{fig:earthquake} containing two dominant MAGs of the Causal Network \textit{Earthquake}.

\begin{figure}[!h]
    \centering
    \begin{minipage}{0.6\textwidth}
    \includegraphics[scale=0.9]{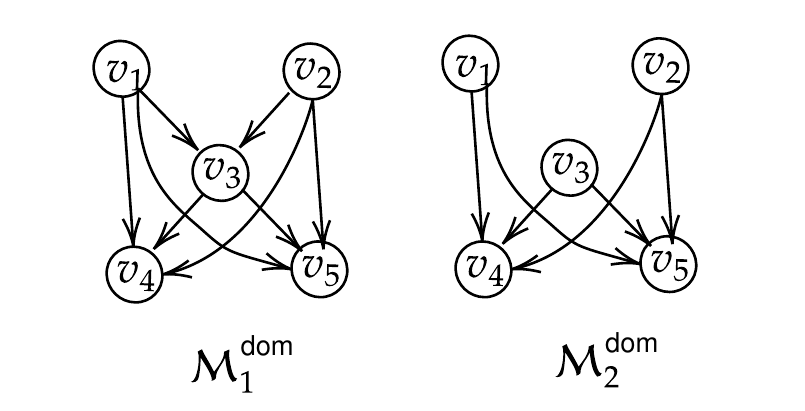}
    \end{minipage}
    \begin{minipage}{0.34\textwidth}
    \caption{Dominant MAGs of the causal network \textit{Earthquake} constructed using the described procedure.}
    \label{fig:earthquake}
    \end{minipage}
\end{figure}

\smallskip
\noindent \textbf{Parameters}. We set number of entities $M = 40$, number of clusters $k=2$, $\alpha = 0.60, \beta = 0.20$, and dominant MAG parameter $\gamma = 0.90$ for both the clusters. For the synthetic data generated using Erd\"os-R\'enyi model, we use $n = 10$, probability of edge $0.3$. 

\smallskip
\noindent \textbf{Evaluation of Clustering.} First, we focus on recovering the clustering using Algorithm~\abBoundedDegree. As a baseline, we employ the well-studied FCI algorithm based on purely observational data~\citep{spirtes2000causation}. After recovering the PAGs corresponding to the MAGs using FCI, we cluster them by constructing a similarity graph (similar to \abBoundedDegree) defined on the set of entities. For Algorithm~\abBoundedDegree, we first construct a sample $S$, and perform various interventions based on the set $S$ for every entity to obtain the clusters. We also implemented another baseline algorithm (\textsc{Greedy}) that uses interventions, based on a greedy idea that selects nodes to set $S$ in Algorithm~\abBoundedDegree by considering nodes in increasing order of their degree in the PAGs returned by FCI. We use this ordering to minimize the no. of interventions as we intervene on every node in $S$ and their neighbors.

\smallskip
\noindent \textbf{Construction of Clusters from FCI Output.} Our first focus is on recovering the true clustering using Algorithm~\abBoundedDegree. As a baseline, we employ the well-studied FCI algorithm~\citep{spirtes2000causation}. We know that FCI returns the Partial Ancestral Graph(PAG) corresponding to the causal MAG using only the observational data. After recovering the PAGs corresponding to the MAGs using FCI, we cluster them by constructing a weighted graph (similar to Algorithm~\abBoundedDegree) defined on the set of entities. For every pair of entities $i, j$, we calculate the number of nodes $n_{ij}$ that share the same neighborhood using the PAGs associated with them, and assign the weight of the edge as $n_{ij}$. This weight captures the similarity between two entities, and whether they belong to the same cluster or not. Now, we use minimum-$k$-cut algorithm to partition the set of entities into $k$ components or clusters. In Algorithm~\abBoundedDegree, we first construct a sample $S$, and perform various interventions based on the set $S$ for every entity to finally obtain the $k$ clusters.

\smallskip
\noindent \textbf{Metrics}. We use the following standard metrics for comparing the clustering performance: \textit{precision} (fraction of pairs of entities correctly placed in a cluster together to the total number of pairs placed in a cluster together), \textit{recall} (fraction of pairs of entities correctly placed in a cluster together to the total number of pairs in the same ground truth clusters), and \textit{accuracy} (fraction of pairs of entities correctly placed or not placed in a cluster to the total number of pairs of entities).
\begin{table*}[!h]
\centering
\small
  \begin{tabular}{|l|l|l|l|l|l|l|c|}
    \toprule
      Causal &
      \multicolumn{3}{c|}{FCI} &
      \multicolumn{3}{c|}{\abBoundedDegree(Alg.~\ref{alg:boundlatents})} &
      \multicolumn{1}{c|}{Maximum} \\
     Network &  {Precision} & {Recall} & {Accuracy} & {Precision} & {Recall} & {Accuracy} & {\# Interventions}\\
      \midrule
    \textit{Earthquake} & $0.57\pm0.18$ & $0.94\pm 0.013$ &$0.58\pm0.18 $  &$0.78\pm0.24 $  & $0.92\pm0.03 $ & $0.77\pm0.23 $ & 4  \\
    \textit{Survey} & $0.62\pm0.21$ & $0.94\pm0.013$ & $0.62\pm 0.2$ & $0.64\pm 0.23$ & $0.97\pm 0.02$ & $0.63 \pm 0.23$ & 5\\
    \textit{Asia} & $0.57\pm 0.18$ & $0.94\pm0.013$ & $0.58\pm 0.18$ & $0.92\pm 0.14$ & $0.95\pm 0.03$ & $0.91 \pm 0.14$ & 5 \\
    \textit{Sachs} & $0.52 \pm 0.12$ & $0.94\pm 0.01$ & $0.52\pm 0.12$ & $0.89\pm 0.20$ & $0.96 \pm 0.02$& $0.88\pm 0.19$ &  6\\
    \textit{Erd\"os-R\'enyi} & $0.62\pm0.21$ & $0.94\pm0.02$ & $0.62\pm 0.21$ & $1.0\pm 0.00$ & $0.95\pm 0.02$ & $0.97\pm 0.013$ & 6 \\
    \bottomrule
  \end{tabular}
  \caption{In this table, we present the precision, recall and accuracy values obtained by our Algorithm~\abBoundedDegree and using FCI. Each cell includes the mean value along with the standard deviation computed over 10 runs.  The last column represents the maximum number of interventions per entity including both Algorithms~\abBoundedDegree and~\Recovery.}
  \label{table:alpha_beta}
\end{table*}

\smallskip
\noindent \textbf{Sample Set $S$ Size.} For Algorithm ~\abBoundedDegree, we use different sample sizes $S$ ranging from $1$ to $3$. In Figure~\ref{fig:int_vs_sample}, we plot the mean value of the maximum number of interventions per entity with change in sample set size. With increase in sample set size, our Algorithm \abBoundedDegree requires  more interventions (see Lemma~\ref{lem:latent_alpha_beta_appendix}) and we observe the same in Figure~\ref{fig:int_vs_sample}. We chose the smallest size $|S|=1$ in our experiments, as increasing the size will increase the number of interventions but did not lead to much improved clustering results. As a sample set of size $1$ roughly corresponds to around $3$ interventions (across all causal networks), we use that for results presented in Table~\ref{table:alpha_beta}.

\begin{figure}[H]
\centering
    \includegraphics[scale=0.65]{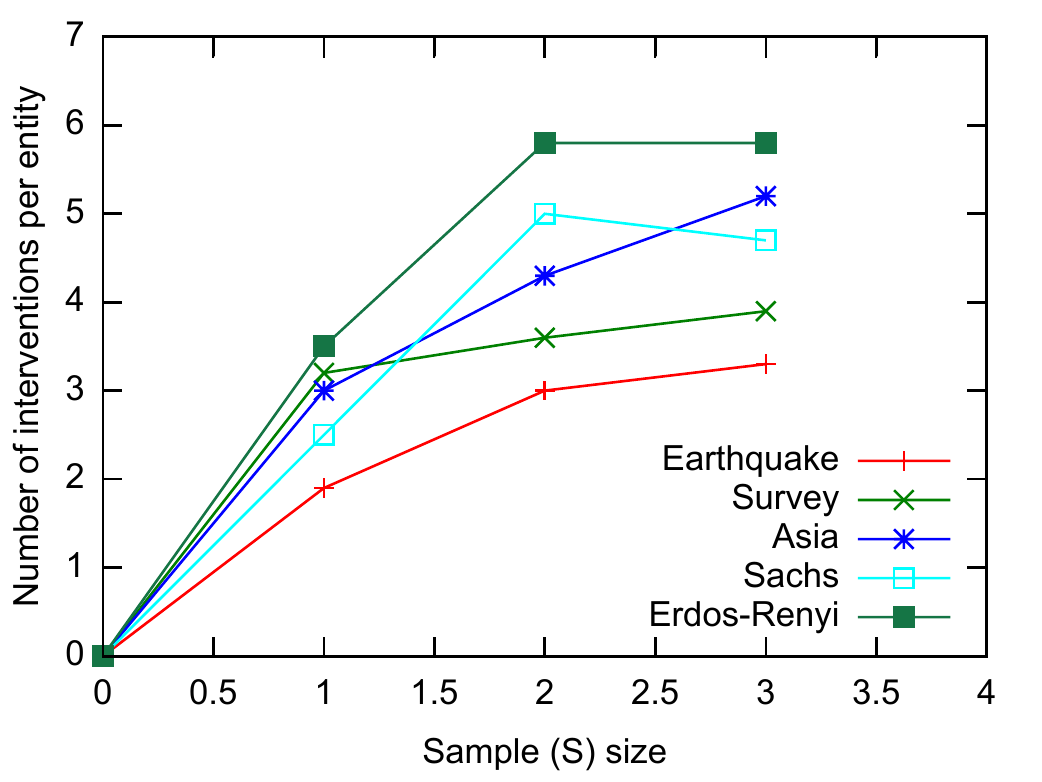}
    \caption{Sample size vs.\ maximum number of interventions per entity used by Algorithm~\abBoundedDegree.}
    \label{fig:int_vs_sample}
\end{figure}


\smallskip
\noindent \textbf{Results.} In Table~\ref{table:alpha_beta}, we compare Algorithm~\abBoundedDegree to FCI on the clustering results. For Algorithm~\abBoundedDegree, we use a sample $S$ of size $1$, and observe in Figure~\ref{fig:int_vs_sample}, that this corresponds to about $3$ interventions per entity. With increase in sample size, we observed that the results were either comparable or better. We observe that our approach leads to considerably better performance in terms of the accuracy metric with an average difference in mean accuracy of about $0.25$. This is because FCI recovers partial graphs, and clustering based on the partial information results in poor accuracy.  Because of the presence of a dominant MAG with in each cluster, we observe that the corresponding entities are always assigned to the same cluster, resulting in high recall for both \abBoundedDegree and FCI. We observe a higher value of precision for our algorithms, because FCI is unable to correctly classify the MAGs that are different from the dominating MAG. 

Algorithm~\abBoundedDegree outperforms the \textsc{Greedy} baseline for the same sample(S) size. For example, on the \textit{Earthquake} and \textit{Survey} causal networks, Algorithm~\abBoundedDegree obtains the mean accuracy values of $0.77$ and $0.63$ respectively, while  \textsc{Greedy} for the same number of interventions obtained an accuracy of only $0.487$ and $0.486$ respectively. For the remaining networks, the accuracy values of  \textsc{Greedy} are almost comparable to our Algorithm~\abBoundedDegree.

After clustering, we recover the dominant MAGs using Algorithm~\Recovery, and observe that the additional interventions needed are bounded by the maximum degree of the graphs (see Theorem~\ref{thm:dominantMAG}). This is represented in the last column in Table~\ref{table:alpha_beta}.  We observe that our \emph{collaborative} algorithms use fewer interventions for dominant MAG recovery compared to the number of nodes in each graph. For example, in the Erd\"os-R\'enyi setup, the number of nodes $n=10$, whereas we use at most $6$ interventions per entity. Thus, compared to the worst-case, cutting the number of interventions for each entity by $40\%$.

%

\section{Conclusion}
We introduce a new model for causal discovery to capture practical scenarios where are multiple entities with different causal structures. Under natural clustering assumption(s), we give efficient provable algorithms for causal learning with atomic interventions in this setup and demonstrate its empirical performance. Our model can be naturally extended to the setting where all interventions are non-adaptive, and we plan to study it as part of future work.  An interesting future direction would be to extend the interventional equivalence between DAGs studied in~\citep{hauser2012characterization, katz2019size} to the setting without the causal sufficiency assumption, similar to~\citep{jaber2020causal, kocaoglu2019characterization}, and exploit that for learning.

\newpage
\nocite{dataset}
\bibliography{references}
\bibliographystyle{plainnat}

\newpage
\appendix
\noindent\rule{\textwidth}{1pt}
\begin{center}
{\Large \sf Appendix}
\end{center}
\noindent\rule{\textwidth}{1pt}

\section{Missing Details from Sections~\ref{sec:intro} and~\ref{sec:model}} \label{app:model}

\begin{figure}[!ht]
 	\centering
	\begin{subfigure}[t]{0.22\textwidth}
		\includegraphics[scale=0.3]{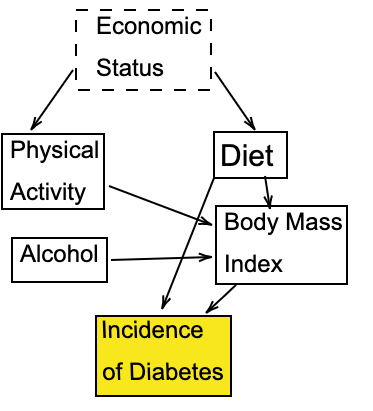}
		\caption{Causal DAG $\DDD_1$}
		\label{fig:g1}	\end{subfigure}
		\hspace*{.5in}
	\begin{subfigure}[t]{0.22\textwidth}
		\includegraphics[scale=0.3]{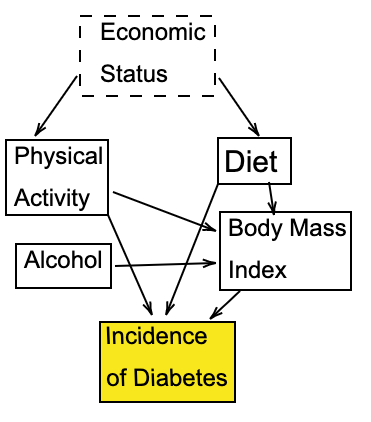}
		\caption{Causal DAG $\DDD_2$}
		\label{fig:g2}
	\end{subfigure}
			\caption{Two possible diabetes incidence graphs for an individual from~\citep{joffe2012causal} differing in the causal edge between \textit{Physical Activity} and~\textit{Incidence of Diabetes}. The observed variables include: \textit{Diet, Body Mass Index (BMI), Physical Activity, Alcohol (consumption), Incidence of Diabetes}, and the unobserved variable (latent) is \textit{Economic Status}. The variable \textit{Incidence of Diabetes} is observable but can't be intervened on, this is not an issue as it has no outgoing edges in the graphs. In this paper, we do not know the underlying causal graphs or which individuals share the same graph. As intervening on variables such as \textit{Diet, BMI} might need expensive and careful experimental organization, we ask the following question -- given a collection of independent entities (in this diabetes example, they can refer to a collection of people), can we collaboratively learn each entity's causal graphs while minimizing the number of interventions per entity?}
\end{figure}

\begin{figure}[!ht]
    \centering
    \includegraphics[scale=0.30]{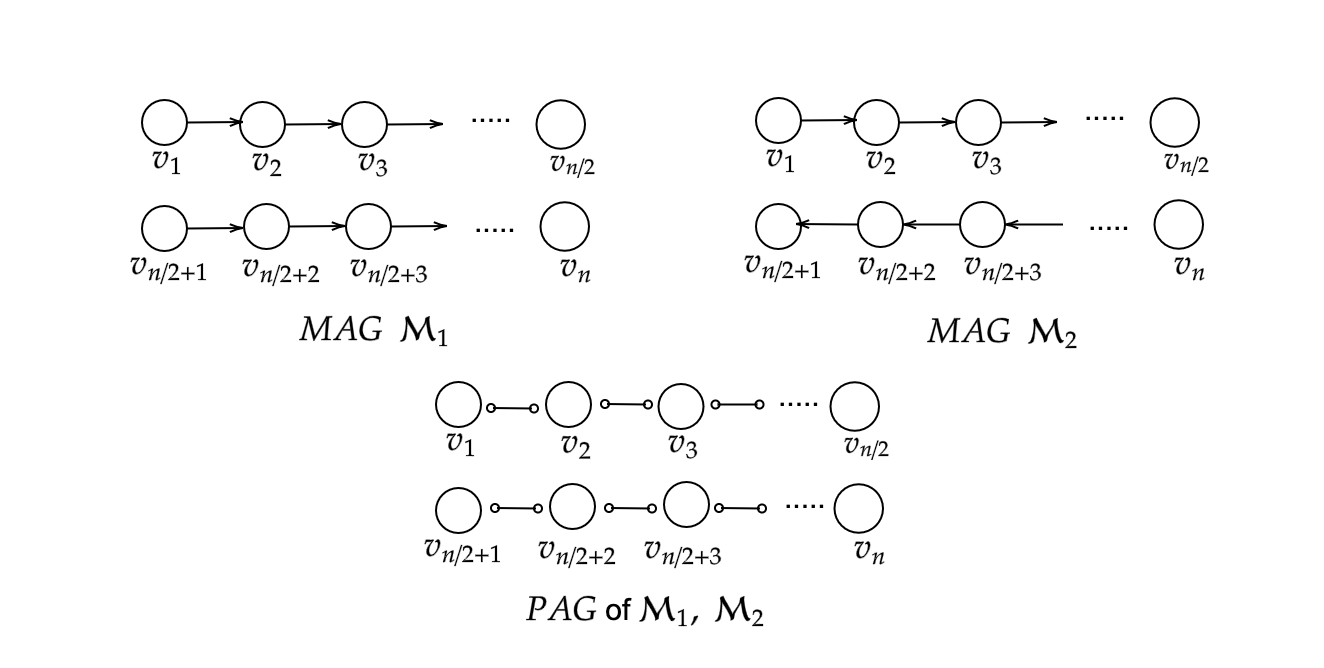}
    \caption{An example of MAGs $\MMM_1$ and $\MMM_2$ with large distance $d(\MMM_1,\MMM_2)$ but generating the same PAG.}
    \label{fig:pag_distance}
    \vspace*{1ex}
\end{figure}

\section{Helper Routines}\label{app:recovery_naive}

\begin{claim}\label{cl:intervention_ancestor}
 Suppose $\DDD_i$ is the DAG and $\MMM_i$ is the corresponding MAG for some entity $i \in [M]$. Then, $u \notindep v \mid \doo(u)$ iff $u$ is an ancestor of $v$ in the graph $\DDD_i$.
\end{claim}
\begin{proof}
We follow a proof similar to Lemma 1 in~\citep{neurips17}. If $u$ is an ancestor of $v$ in the graph $\DDD_i$ using the path $\pi_{uv}$, then, in the mutilated graph corresponding to $\doo(u)$, the path $\pi_{uv}$ remains intact. From d-separation~\citep{pearl}, $\pi_{uv}$ can only be blocked by conditioning on one of the nodes that are not end points. As we do not condition on any variables in the CI-test $u \indep v \mid \doo(u)$ and therefore do not block the path $\pi_{uv}$, we have $u \notindep v \mid \doo(u)$.

Now, we consider the other direction. If $u \notindep v \mid \doo(u)$, then, there is at least a path $\pi_{uv}$ between $u$ and $v$ that is not blocked. In the mutilated graph corresponding to the interventional distribution $\doo(u)$, the incoming edges into the node $u$ are removed. In the path $\pi_{uv}$, the edge incident on $u$ is an outgoing edge. If there is a collider on $\pi_{uv}$, we have blocked the path by not conditioning on it (from d-separation). As the path is not blocked, it implies that there is no collider on the path. Therefore, the path $\pi_{uv}$ is a directed path from $u$ to $v$. Hence, the claim.
\end{proof}

\begin{claim}\label{cl:outnbr}
Given an entity $i \in [M]$, and a node $u \in V$, Algorithm~\IdentifyNbr identifies all outgoing edges of $u$ in $\MMM_i$ ($\out{i}{u}$) correctly using an intervention on $u$.
\end{claim}
\begin{proof}
We know that $\PPP_i = (V,\hat{E}_i)$ represents the partial ancestral graph of $\MMM_i$. We observe that any outgoing edge $(u, v)$ incident on a node $u$ in the PAG $\PPP_i$ can be of the form $u \circtailcirc v$ or $u \rightarrowcirc v$. Otherwise, we already know that the edge is not an outgoing edge from $u$. We claim that we can identify an outgoing edge $(u, v)$ from a node $u$ correctly, if CI-test returns $u \notindep v \mid \doo(u)$ for every $v \in \DD{i}{u}$ satisfying the condition mentioned above. From Claim~\ref{cl:intervention_ancestor}, we have that $u \notindep v \mid \doo(u)$ iff  $u$ is an ancestor of $v$ in $\DDD_i$, which implies $u \rightarrow v$ is present in $\MMM_i$ and $v \in \out{i}{u}$. 
\end{proof}

\begin{claim}\label{cl:bidir}
Given an entity $i \in [M]$, and a node $u \in V$, Algorithm~\Identifybidir identifies all bidirected edges incident on $u$ in $\MMM_i$ ($\bi{i}{u}$) correctly using atomic interventions on all nodes in $\DD{i}{u}$.
\end{claim}
\begin{proof}
We observe that any bi-directed edge $(u, v)$ incident on a node $u \in V$ in the PAG $\PPP_i$ can be of the form  $u \circtailcirc v$ or $u \leftarrowcirc v$ or $u \rightarrowcirc v$. Otherwise, we already know that the edge is not a bi-directed edge incident at $u$. In Algorithm~\Identifybidir, for every neighbor $v$ of $u$ in the PAG $\PPP_i$ satisfying the above condition, we check if $u \indep v \mid \doo(u)$ and $v \indep u \mid \doo(v)$ is satisfied. From Claim~\ref{cl:intervention_ancestor}, we know that if $u \notindep v \mid \doo(u)$, then $u$ is an ancestor of $v$ in $\DDD_i$ (similarly, $v$ is an ancestor of $u$ in $\DDD_i$ if $u \notindep v \mid \doo(v)$). So, if $u \indep v \mid \doo(u)$ and $u \indep v \mid \doo(v)$, then, $u$ is not an ancestor of $v$ or vice-versa, which implies $u \leftrightarrow v$ is present in $\MMM_i$, i.e., $v \in \bi{i}{u}$. As we perform an intervention for every neighbor of $u$ in $\UUU_i$, we have the claim.
\end{proof}

\begin{lemma}[Lemma~\ref{lem:naive_main} restated] \label{lem:naive}
Algorithm~\RecoverG recovers all edges of $\MMM_i$, for an entity $i \in [M]$ using $n$ atomic interventions.
\end{lemma}
\begin{proof}
Given an entity $i \in [M]$, we obtain the partial ancestral graph $\PPP_i$ from observational data. Using Algorithm~\RecoverG, we create interventions for every node $u \in V$.  For every node $u$, we correctly identify all the outgoing neighbors of $u$ using Algorithm~\IdentifyNbr  (Claim~\ref{cl:outnbr}) and all the bidirected edges using Algorithm~\Identifybidir (Claim~\ref{cl:bidir}). Therefore, we have recovered all edges of $\MMM_i$ using $n$ atomic interventions.
\end{proof}

\begin{proposition}\label{prop:lb_naive_appendix}[Proposition~\ref{prop:lb_naive} restated]
There exists a causal MAG $\MMM$ such that every adaptive or non-adaptive algorithm requires $n$ many {atomic} interventions to recover $\MMM$. 
\end{proposition}
\begin{proof}
Suppose the set of nodes of an unknown MAG $\MMM$ is given by $V = \{ v_1, v_2, \cdots v_n \}$. We denote $\ALG$ by any adaptive or non-adaptive \emph{deterministic} algorithm that recovers $\MMM$ using the set of interventions $\SSS \subseteq V$. For the sake of contradiction, let $\ALG$ recover $\MMM$ correctly and $v_i$ be the vertex that has not been intervened on, i.e., $v_i \not\in \SSS$. 

\begin{figure}
    \centering
    \includegraphics[scale=0.80]{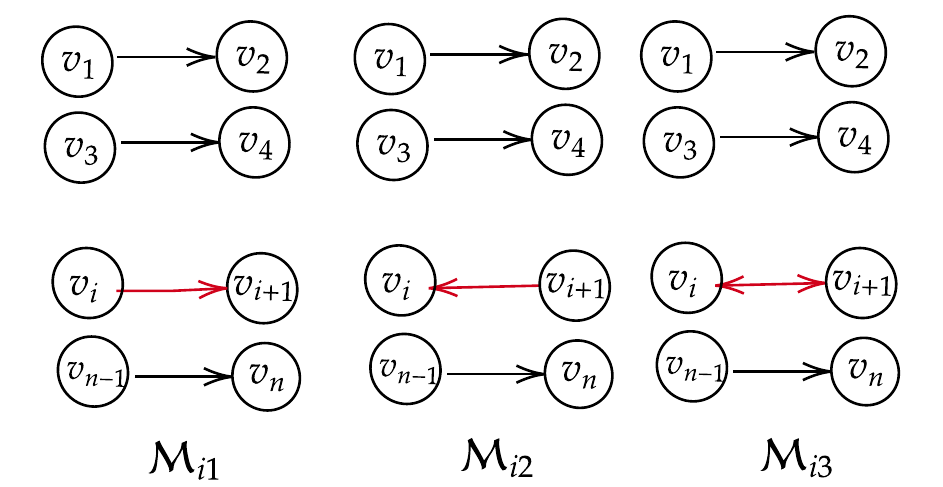}
    \caption{The MAGs used in the proof of Proposition~\ref{prop:lb_naive_appendix}.}
    \label{fig:lb}
\end{figure}
Construct the MAGs $\MMM_{i1}, \MMM_{i2}, \MMM_{i3}$ with edges $E_{i1} = \{ v_1 \rightarrow v_2, v_3 \rightarrow v_4,  \cdots, v_i\rightarrow v_{i+1}, \cdots, v_{n-1}\rightarrow v_n \}, E_{i2} = \{ v_1 \rightarrow v_2, v_3 \rightarrow v_4,  \cdots, v_i\leftarrow v_{i+1}, \cdots, v_{n-1}\rightarrow v_n \}, E_{i3} = \{ v_1 \rightarrow v_2, v_3 \rightarrow v_4,  \cdots, v_i\leftrightarrow v_{i+1}, \cdots, v_{n-1}\rightarrow v_n \}$ respectively (see Figure~\ref{fig:lb}). 

Upon termination, $\ALG$ will have recovered one of the MAGs $\MMM_{i1}$, $\MMM_{i2}$ or $\MMM_{i3}$. As $v_i \not\in \SSS$, we will argue that the true MAG is different from the recovered MAG. We consider two cases:
\begin{enumerate}
    \item If $v_{i+1} \in \SSS$. First, we observe that for all three MAGs $\MMM_{i1}, \MMM_{i2}$ and $\MMM_{i3}$, the CI-test $v_i \notindep v_{i+1}$. For MAGs $\MMM_{i1}$ and $\MMM_{i2}$, we have $v_{i} \indep v_{i+1} \mid \doo(v_{i+1})$ while $v_i \notindep v_{i+1} \mid \doo(v_{i+1})$ for MAG $\MMM_{i2}$. As these are the only possible CI-tests for vertices $v_{i}$ and $v_{i+1}$, the algorithm $\ALG$ cannot differentiate between $\MMM_{i1}$ and $\MMM_{i3}$. If $\ALG$ recovers $\MMM_{i1}$, then, we can set $\MMM$ to be $\MMM_{i3}$. This is a contradiction.
    \item If $v_{i+1} \not\in \SSS$. We observe that for all three MAGs $\MMM_{i1}, \MMM_{i2}$ and $\MMM_{i3}$, the CI-test $v_i \notindep v_{i+1}$, and it is the only possible CI-test involving vertices $v_{i}$ and $v_{i+1}$. Therefore, the algorithm $\ALG$ cannot differentiate between $\MMM_{i1}, \MMM_{i2}$  and $\MMM_{i3}$. If $\ALG$ recovers $\MMM_{i1}$, then, we can set $\MMM$ to be $\MMM_{i2}, \MMM_{i3}$ and similarly for other cases. This is a contradiction.
\end{enumerate}

Therefore, to recover $\MMM \in \{ \MMM_{i1}, \MMM_{i2}, \MMM_{i3}\}$ correctly, we must have $v_i \in \SSS$. As $i$ is chosen arbitrarily, and for every $i$ we can construct the MAGs $\MMM_{i1}, \MMM_{i2}, \MMM_{i3}$, such that any adaptive or non-adaptive deterministic algorithm requires interventions on every node.

We can extend the proof to include \emph{randomized} algorithms, with success probability strictly greater than $1/2$, by observing that when $v_i \not\in \SSS$, $\ALG$ has at least two MAGs among $\MMM_{i1}, \MMM_{i2}, \MMM_{i3}$ that it cannot differentiate (as argued using two cases above). 
\end{proof}

\section{Discovery under $(\alpha,\beta)$-Clustering}\label{app:alphabeta}

In this section, we present an algorithm that recovers the underlying clusters $\Copt_1, \Copt_2, \cdots, \Copt_k$ provided they satisfy $(\alpha, \beta)$-clustering property. After recovering the clusters, in Section~\ref{app:alphabeta_recovery}, we give an algorithm that recovers an approximate MAG for every entity with only few additional interventions.

Firstly, using the next lemma, we show that the threshold used by Algorithm~\abBoundedDegree correctly identifies whether two entities belong to the same true cluster or not. This implies that our algorithm \abBoundedDegree recovers the clusters with high probability. 

\begin{lemma}[Lemma~\ref{lem:latent_alpha_beta} Restated]
\label{lem:latent_alpha_beta_appendix}
If the underlying MAGs $\MMM_1,\dots,\MMM_M$ satisfy $(\alpha,\beta)$-clustering property with true clusters $\Copt_1,\dots,\Copt_k$ and have maximum undirected degree $\Delta$. Then, the Algorithm~\abBoundedDegree recovers the clusters $\Copt_1,\dots,\Copt_k$ with  probability at least $1-\delta$. Every entity $i \in [M]$ uses at most $4(\Delta+1) \log(M/\delta)/(\alpha-\beta)^2$ many atomic interventions.
\end{lemma}
\begin{proof}
Let $\Count(i, j) = \sum_{u \in S} \mathbf{1}\{ \NN{i}{u} = \NN{j}{u} \}$ for distinct entities $i,j$. If $i, j$ belong to the same true cluster $\Copt_t$ for some $t \in [k]$, we have :
$$\E[\Count(i, j)] = \E\left[\sum_{u \in S} \mathbf{1}\{ \NN{i}{u} = \NN{j}{u} \} \right] \geq (1-\beta)|S|$$

Using Hoeffding's inequality, with probability at least $1-\exp{(-\frac{\Lambda^2}{2|S|})}$
\[ \Count(i, j) \geq \E[\Count(i,j)] - \frac{\Lambda}{2}\]

If $i, j$ belong to different true clusters, then, we have : 

\[ \E[\Count(i, j)] = \E\left[\sum_{u \in S} \mathbf{1}\{ \NN{i}{u} = \NN{j}{u} \} \right] \leq (1-\alpha)|S|\]

Using Hoeffding's inequality, with probability at least $1-\exp{(-\frac{\Lambda^2}{2|S|})}$
\[ \Count(i, j) < \E[\Count(i,j)] + \frac{\Lambda}{2}\]

Set $\Lambda = |S|(\alpha - \beta)$ and $|S| = \frac{4\log M/\delta}{(\alpha-\beta)^2}$.

Using union bound for every pair of entities in $[M]$, we have with probability at least $1-\delta$:

\[ \text{if entities } i, j \in \Copt_t \mbox{ (belong to the same true cluster)}: \Count(i, j) \geq \left( 1- \frac{\alpha+\beta}{2} \right)|S| \  \text{ and }\]
\[ \text{if} \ \text{ entities } i, j \not\in \Copt_b  \  \forall b \in [k] \mbox{ (do not belong to the same true cluster)}: \Count(i, j) < \left( 1- \frac{\alpha+\beta}{2} \right)|S|\]

Therefore, every pair of entities from same true cluster satisfy the condition that  $\Count$ value is larger than $(1-\frac{\alpha+\beta}{2})|S|$ and will include an edge in $\mathcal{P}$, while we do not include an edge between pair of entities from different clusters. The resulting graph $\mathcal{P}$, will have $k$ connected components and Algorithm~\abBoundedDegree~will return the true clusters correctly. 

As we intervene on all the neighbors of every node in $S$, it will increase the interventions for every entity by a multiplicative $\Delta+1$ factor. For an entity $i$, the total number of interventional distributions constructed is 
$$\sum_{u \in S}(1+|\DD{i}{u}|) \leq |S|(\Delta + 1) = 4(\Delta+1) \log(M/\delta)/(\alpha-\beta)^2 \ \text{ as } \max_{i \in [M], w \in V} |\DD{i}{w}| \leq \Delta.$$

\end{proof}

\subsection{Learning Causal Graphs from Clusters}\label{app:alphabeta_recovery}

In this section, we provide additional details about the Algorithm~\Recovery that returns an approximate causal graph for every entity $i \in [M]$.

Consider the cluster $\Copt_a$ for some $a \in [k]$. In the next claim, we show that if the size of $\Copt_a$ is sufficiently large, then, each node $u \in V$ is assigned a large number of entities by \Recovery using the set $T_u$.
\begin{claim}
Consider a cluster $\Copt_a$ such that $|\Copt_a| \geq \frac{8 n\log(nM/\delta)}{(2\gamma_a-1)^2}$. Let $T_u$ denote the set of entities assigned to node $u$ in Algorithm~\ref{alg:graphs_boundlatents}. Then, we have with probability $1-\delta$, $|T_u| \geq \frac{4\log(n M/\delta)}{(2\gamma_a-1)^2}$ for every node $u \in V$.
\end{claim}
\begin{proof}
For a node $u \in V$, and cluster $\Copt_a$, we have:
$$ \E[T_u] = \frac{|\Copt_a|}{n} \geq \frac{8\log(nM/\delta)}{(2\gamma_a-1)^2}.$$

Using Chernoff bound, with probability at least $1-\exp{(-\log(nM/\delta)/(2\gamma_a-1)^2)} \geq 1-{\delta}/{n M}$, we have: 
\[ T_u \geq \frac{\E[T_u]}{2}  \geq \frac{4\log(nM/\delta)}{(2\gamma_a-1)^2}.\]

Applying union bound for every node $u \in V$ and $a \in [k]$, gives us the claim.
\end{proof}

\smallskip
\noindent Consider a partitioning of $\Copt_a$ given by $S^1_a, S^2_a, \cdots S^t_a$ where each $S^i_a$ for any $i \in [t]$ represents the maximal collection of MAGs that are equal. Formally, we have: $$S^i_a = \{ \MMM_p \mid \MMM_p \in \Copt_a \text{ and } \MMM_p = \MMM_q \quad \forall \MMM_q \in S^i_a\}.$$

\smallskip
\noindent Let $|S^\text{dom}_a| \geq |S^i_a|$ for every partition $i \in [t]$ and $\text{dom}_a$ denote an entity in $S^\text{dom}_a$. We define:
$$G_a(u) = \{ j \mid j \in \Copt_a \text{ and } N_i(u) = N_{\text{dom}_a}(u) \} \text{ and } B_a(u) = \Copt_a \setminus G_a(u). $$

We can observe that:
\[ |G_a(u)| \geq |S^\text{dom}_a| \text{ and } |B_a(u)| \leq  |\Copt_a| - |S^\text{dom}_a|.\]

Conditioned on the previous claim that each set $T_u$ for all $u \in V$ is large, we argue that for any pair of entities $i, j \in \Copt_a$ where $\MMM_i = \MMMdom_a$, and $\MMM_j \neq \MMMdom_a$, the $\NCount$ value calculated by \Recovery of entity $i$  for the node $u$ is always larger than that of entity $j$. Intuitively, after assigning the entities to nodes, we observe that for every node $u \in V$, the set $T_u$ contains a large number of entities with the dominant MAG, i.e., $|T_u \cap G_a(u)|$ is large. Because dominant MAGs share the same neighborhood (as they represent the same graph), we can show that the $\NCount$ value of dominant MAG is larger than any other MAG in the cluster. We formalize this statement using the following lemma.

\begin{lemma}\label{lem:dominantCount_appendix}
For every $a \in [k], u \in V$ and any pair of entities $i, j \in \Copt_a$ that satisfy $i \in G_a(u)$ and $j \in B_a(u)$, we have with probability $1-\delta$,  \[\NCount(i, u) > \NCount(j, u).\]
\end{lemma}
\begin{proof}
From Algorithm~\ref{alg:graphs_boundlatents}, we know that $\NCount(i, u) = \sum_{j \neq i, j \in T_u} \mathbf{1} \{ N_i(u) = N_j (u) \}$ for an entity $i \in T_u$ and a node $u \in V$. Consider the case $i \in G_a(u)$. Then, we have:
\begin{align*}
    \E[\NCount(i, u)] &= \E\left[\sum_{j \neq i, j \in T_u} \mathbf{1} \{ N_i(u) = N_j (u) \}  \right] \\
                      &= \E\left[\sum_{j \neq i, j \in T_u} \mathbf{1} \{ N_{\text{dom}_a}(u) = N_j (u) \}  \right] \\
                      &= \E[|T_u \cap G_a(u)|]\\
                      &\geq \frac{|S^\text{dom}_a|}{n} = \frac{|S^\text{dom}_a|}{|\Copt_a|} \cdot \frac{|\Copt_a|}{n}= \gamma_a \cdot \frac{|\Copt_a|}{n}
\end{align*}
Using Hoeffding's inequality, with probability at least $1-\exp{(-\frac{\Lambda^2}{2|T_u|})}$
\[ \NCount(i, u) \geq \E[\NCount(i,u)] - \frac{\Lambda}{2} \geq \gamma_a \cdot \frac{|\Copt_a|}{n} - \frac{\Lambda}{2} \]
If $i \in B_a(u)$, then, we have : 
\begin{align*}
    \E[\NCount(i, u)] &= \E\left[\sum_{j \neq i, j \in T_u} \mathbf{1} \{ N_i(u) = N_j (u) \} \right] \\
    &= \E\left[\sum_{j \neq i, j \in T_u \cap G_a(u)} \mathbf{1} \{ N_i(u) = N_{\text{dom}_a}(u) \} + \sum_{j \neq i, j \in T_u \cap B_a(u)} \mathbf{1} \{ N_i(u) = N_{j}(u) \}  \right] \\
    &= \E\left[\sum_{j \neq i, j \in T_u \cap B_a(u)} \mathbf{1} \{ N_i(u) = N_{j}(u) \}  \right] \\
    &= \E[|T_u \cap B_a(u)|]\\
    &\leq \frac{|\Copt_a| - |S^{\text{dom}}_a|}{n} = \left( 1 - \frac{|S^{\text{dom}}_a|}{|\Copt_a|} \right) \cdot \frac{|\Copt_a|}{n} = (1-\gamma_a)\cdot \frac{|\Copt_a|}{n}
\end{align*}
Using Hoeffding's inequality, with probability at least $1-\exp{(-\frac{\Lambda^2}{2|T_u|})}$
\[ \NCount(i, u) < \E[\Count(i,u)] + \frac{\Lambda}{2} <  (1-\gamma_a)\cdot \frac{|\Copt_a|}{n} + \frac{\Lambda}{2}\]
Set $\Lambda = \frac{|\Copt_a|}{n}(2\gamma_a-1)$ and $|T_u| \geq \frac{4\log(n M/\delta)}{(2\gamma_a-1)^2}$. Then, for any pair of entities $i, j \in \Copt_a$ such that $i \in G_a(u)$ and $j \in B_a(u)$, we have, with a probability $1-\delta/n M^2$: $$\NCount(i, u) > \NCount(j, u).$$ 

Using union bound for every pair of entities in $[M]$ and $u \in V$, with probability at least $1-\delta$, we have the final claim.
\end{proof}

From the previous Lemma~\ref{lem:dominantCount_appendix}, we know that $\NCount$ values are always larger for the dominant MAG partition, and therefore merging the neighborhoods of all the nodes gives us the dominant MAG. As dominant MAG is within a distance of at most $\beta \cdot n$ from every MAG in the cluster, the dominant MAG returned is a sufficiently good approximation of the true MAG. We formalize this using the following statement. 

\begin{theorem}[Theorem~\ref{thm:dominantMAG} Restated]\label{thm:dominantMAG_appendix}
Suppose $\MMM_1, \MMM_2, \cdots \MMM_M$ satisfy $(\alpha, \beta)$ clustering property. If $\Copt_a = \Omega(\frac{n\log(n/M\delta)}{(2\gamma_a-1)^2})$ for all $a \in [k]$, then, Algorithm~\Recovery recovers graphs $\hat{\MMM}_1, \cdots \hat{\MMM}_M$ such that for every entity $i \in [M]$, we have $d(\MMM_i, \hat{\MMM}_i) \leq \beta n$  with probability $1-\delta$. Moreover, every entity uses at most $ (\Delta+1) + \frac{4(\Delta+1) \log(M/\delta)}{(\alpha-\beta)^2}$ many atomic interventions. 
\end{theorem}
\begin{proof}
From Lemma~\ref{lem:dominantCount_appendix}, we have that $\NCount(i, u) > \NCount(j, u)$, which implies $u_{\max} \in G_a(u)$. Using Algorithm~\ref{alg:graphs_boundlatents}, every entity $i$ in the cluster $\Copt_a$ is assigned the graph $\hat{\MMM}_i  = \MMMdom_a$. From the definition of $(\alpha, \beta)-$clustering property, we have that all entities $i \in \Copt_a$ are such that $d(\MMM_i, \hat{\MMM}_i) = d(\MMM_i, \MMMdom_a) \leq \beta n$.

Using Algorithm~\ref{alg:graphs_boundlatents} we assign every entity to a single node $u \in V$, and perform at most $\Delta + 1$ interventions to identify all the neighbors of $u$ for every entity in $T_u$. Therefore, we perform at most $\Delta+1$ interventions per entity. For obtaining clusters, from Lemma~\ref{lem:latent_alpha_beta}, we know that every entity performs at most $\frac{4(\Delta+1) \log(M/\delta)}{(\alpha-\beta)^2}$ interventions. Hence, the theorem.
\end{proof}

\section{Discovery under $\alpha$-Clustering Property} \label{app:graph_recovery}

In this section we provide additional details about learning causal graphs under $\alpha$-clustering property. 

\subsection{From $\alpha$-Clustering to Learning Causal Graphs} \label{app:meta}
 
\begin{lemma}\label{lem:cluster_to_graph_gen_appendix}
Suppose there is an Algorithm $\mathcal{A}$ that recovers the true clusters $\Copt_1, \Copt_2, \cdots, \Copt_k$ of the underlying MAGs $\MMM_1, \MMM_2, \cdots, \MMM_M$ satisfying $\alpha$-clustering property such that every entity $i \in [M]$ uses at most $f(M)$ interventions. Then, there is an algorithm that can learn all the MAGs $\MMM_1, \MMM_2, \cdots, \MMM_M$ such that every entity $i \in [M]$ uses at {most $f(M) + \lceil n/\Upsilon \rceil$ many  interventions, where $\Upsilon = \min_{b \in [k]} \Copt_b$.}
\end{lemma}
\begin{proof}
Consider a cluster $\Copt_b$ for $b \in [k]$. As the mapping $\phi$ assigns every entity at most $\lceil n/|\Copt_b| \rceil$ many nodes to intervene on, we have that every entity in $\Copt_b$ uses at most $\lceil n/|\Copt_b| \rceil$ additional interventions. Therefore, over all true clusters, every entity  uses at most $f(M) + \lceil n/\Upsilon \rceil$ many interventions. 

Consider any cluster $\Copt_b$. The mapping $\phi$ in the {Meta-Algorithm} is well-defined and satisfies the claim that for every node $u \in V$, there exists an entity in $\Copt_b$ for which we construct an interventional distribution $\doo(u)$. Therefore, for  cluster $\Copt_b$, we have $n$ interventional distributions one for every node in $V$, and we use Algorithm~\RecoverG to learn the MAG for this cluster (i.e., MAG for all the entities in $\Copt_b$). Repeating this for every cluster $\Copt_1, \Copt_2, \cdots, \Copt_k$, we obtain all the MAGs $\MMM_1, \MMM_2, \cdots, \MMM_M$.
\end{proof}

\smallskip
\noindent From the above lemma, we have the immediate corollary:
\begin{corollary}[Corollary~\ref{cor:cluster_to_graph} restated]\label{cor:cluster_to_graph_appendix}
Suppose there is an Algorithm $\mathcal{A}$ that recovers the true clusters $\Copt_1, \Copt_2, \cdots, \Copt_k$ of the underlying MAGs $\MMM_1, \MMM_2, \cdots, \MMM_M$ satisfying the $\alpha$-clustering property such that every entity $i \in [M]$ uses at most $f(M)$ interventions. Suppose $\min_{b \in [k]} |\Copt_b| \geq n$. Then, there is an algorithm that can learn all the MAGs $\MMM_1, \MMM_2, \cdots, \MMM_M$ such that every entity $i \in [M]$ uses at most $f(M) + 1$  many interventions.
\end{corollary}

\subsection{Discovery without Latents}\label{app:nolatents}

In this section, we present a randomized algorithm that recovers (with high probability) all the $M$ MAGs $\MMM_1,\dots,\MMM_M$ when the underlying data generating process for each of these entities do not have any latents (i.e., causal DAGs $\DDD_1,\dots,\DDD_M$ satisfy causal sufficiency). This translates into the fact that the MAGs $\MMM_1,\dots,\MMM_M$ do not have bidirected edges. 

We first make the observation that to identify that two graphs, say $\MMM_i$ and $\MMM_j$ belong to different clusters, it suffices to find a node $u$ from the node-difference set $\diff(\MMM_i,\MMM_j)$ and checking their  outgoing neighbors using Algorithm~\IdentifyNbr. We argue that, with probability at least $1-\delta$, we can identify one such node $u \in \diff(\MMM_i,\MMM_j)$ by sampling $2\log (M/\delta)/\alpha$ nodes uniformly from $V$ as $|\diff(\MMM_i,\MMM_j)| = d(\MMM_i, \MMM_j) \geq \alpha n$.

In Algorithm~\NoLatents, we obtain a sample of nodes $S$ and construct interventional distribution for every entity in $[M]$, and for every node in $S$. After finding the outgoing neighbors for every entity $i$ and node in $S$, we construct a graph $\mathcal{P}$ on entities (i.e., the node set of $\mathcal{P}$ is $[M]$). We include an edge between two entities if they share the same outgoing neighbors for every $u \in S$. This ensures that every entity is connected only to the entities belonging to the same true cluster, and we return the connected components in $\mathcal P$ as our clusters.

\begin{algorithm}[!ht]
\caption{\NoLatents}
\label{alg:nolatents}
\begin{algorithmic}[1]
\STATE \textbf{Input:} $\alpha > 0$, confidence parameter $\delta > 0$, PAGs $\UUU_1,\dots,\UUU_M$ of $M$ entities.
\STATE \textbf{Output:} Partition of $[M]$ into clusters
\STATE Let $S$ denote a uniform sample of $\frac{2\log M/\delta}{\alpha}$ nodes from $V$ selected with replacement.
\FOR{every entity $i \in [M]$ and $u \in S$}
\STATE $\out{i}{u} \leftarrow \IdentifyNbr(i, u)$
\ENDFOR
\STATE Let $\mathcal{P}$ denote an empty graph on set of entities $[M]$
\FOR{every pair of entities $i, j$}
\IF{$\out{i}{u} = \out{j}{u}$  and $\D{i}(u) = \D{j}(u)$ for every $u \in S$}
\STATE Add an edge between entities $i$ and $j$ in $\mathcal{P}$
\ENDIF
\ENDFOR
\STATE Return connected components in $\mathcal{P}$
\end{algorithmic}
\end{algorithm}

\begin{claim}\label{cl:separation}
Let $S$ denote a set of $2\log(M/\delta)/\alpha$ nodes sampled with replacement uniformly from $V$. Then, for every pair of entities $i,j$ that belong to different true clusters, we have with probability at least $1-\delta$,  $u \in \diff(\MMM_i, \MMM_j)$  for some $u \in S$.  
\end{claim}
\begin{proof} Let $S$ denote a set of sampled nodes such that $|S| = 2\log(M/\delta)/{\alpha}$. Therefore, we have
\begin{align*}
   \Pr_{u \sim V}[ u \in \diff(\MMM_i, \MMM_j)] &\geq \alpha, \mbox{ and } \\
   \Pr_{S \sim V}[\ \forall u \in S \ : u \not\in \diff(\MMM_i, \MMM_j) ] &\leq (1-\alpha)^{|S|} \\ 
   &\leq e^{-\alpha |S|} \leq \frac{\delta}{M^2}.  
\end{align*}

Using union bound for every pair of entities in $[M]$ that belong to two different clusters, we have:
\[ \forall i, j \in [M], \ \Pr_{S \sim V}[\forall \ u \in S, u \not\in \diff(\MMM_i, \MMM_j)] \leq \delta. \]

Therefore, for every pair of entities $ i, j \in [M]$ belonging to different true clusters, there exists $u \in S$ such that:
\[ \Pr[u \in \diff(\MMM_i, \MMM_j)] \geq 1-\delta. \]
\end{proof}

\begin{lemma}\label{lem:nolatent_app}
Assume causal sufficiency. If MAGs $\MMM_1,\dots,\MMM_M$ satisfy $\alpha$-clustering property with true clusters $\Copt_1,\dots,\Copt_k$, then  Algorithm~\NoLatents    exactly recovers the clusters $\Copt_1,\dots,\Copt_k$ with probability at least $1-\delta$. Every entity $i \in [M]$ uses $2 \log(M/\delta)/\alpha$ many atomic interventions.
\end{lemma}
\begin{proof}
Consider two entities $i,j$ and their corresponding MAGs $\MMM_i$ and $\MMM_j$ respectively. We first observe that if the PAGs of these two entities are different then they belong to different clusters. Now consider the case where the PAGs for both these entities are the same, i.e.,  $\UUU_i = \UUU_j$.

Now if $i$ and $j$ belong to different true clusters, then we claim that it suffices to find a  node $u$ from the node-difference set $\diff(\MMM_i,\MMM_j) = \{ u \mid \NN{i}{u} \neq \NN{j}{u} \}$ to notice this fact. As there are no latents (causal sufficiency), we can identify whether $u \in \diff(\MMM_i, \MMM_j)$, by checking only the outgoing neighbors of $u$ for entities $i, j$, i.e., $\diff(\MMM_i, \MMM_j) = \{ u \mid \out{i}{u} \neq \out{j}{u} \} $. When we identify such a node $u$, the set of outgoing neighbors of node $u$ are different for entities $i, j$, and therefore must belong to different true clusters (by $\alpha$-clustering property). In order to identify at least one node $u \in \diff(\MMM_i, \MMM_j)$, we use sampling. 

Let $S$ denote the set of sampled nodes (with replacement) from $V$ such that $|S| = 2\log(M/\delta)/{\alpha}$. In Algorithm~\NoLatents, we construct interventional distributions for every node $u \in S$, for every entity $i \in [M]$. Using these interventional distributions we obtain the outgoing neighbors of nodes in $S$ using Algorithm~\IdentifyNbr.

From Claim~\ref{cl:separation}, we have that for every pair of entities $ i, j$ belonging to different true clusters, there exists $u \in S$ such that:
\[ \Pr[\out{i}{u} \neq \out{j}{u}] = \Pr[u \in \diff(\MMM_i, \MMM_j)] \geq 1-\delta. \]

This implies that, with probability at least $1-\delta$, for every $i,j$ pair we have the following: in the entity graph $\mathcal{P}$ there would not be an edge between $i,j$ if they belong to different true clusters, and there would be an edge if they belong to the same true cluster. The resulting graph $\mathcal{P}$, will have $k$ connected components and Algorithm~\NoLatents will return the true clusters correctly.

Hence, with probability at least $1-\delta$, we can recover all the true clusters $\Copt_1,\dots,\Copt_k$ using Algorithm~\NoLatents.
\end{proof}

\begin{theorem}\label{thm:nolatent_app}
Assume causal sufficiency. If MAGs $\MMM_1,\dots,\MMM_M$ satisfy $\alpha$-clustering property with true clusters $\Copt_1,\dots,\Copt_k$ then  Algorithm~\NoLatents exactly recovers these clusters with probability at least $1-\delta$. Furthermore, if $\min_{b \in [k]} |\Copt_b| \geq n$,  then there
is an algorithm that exactly learns all these MAGs with probability at least $1-\delta$. Every entity $i \in [M]$ uses $2 \log(M/\delta)/\alpha + 1$ many atomic interventions.
\end{theorem}
\begin{proof}
From Lemma~\ref{lem:nolatent_app}, we can recover the clusters correctly with probability at least $1-\delta$. Using the {Meta-Algorithm} discussed in Section~\ref{app:meta}, we can learn the graphs of every entity with a single additional intervention (see Corollary~\ref{cor:cluster_to_graph_appendix}). This establishes the result.
\end{proof}

\subsection{Discovery with Latents: Bounded Degree MAGs}\label{app:boundedlatents}
Throughout this section, we let : 
\[\Delta = \max_{i \in [M], u \in V} |\D{i}(u)|. \]

We now discuss an algorithm that recovers clusters $\Copt_1, \Copt_2 \cdots, \Copt_k$ using ideas developed in Section~\ref{app:nolatents} but now with latents in the system. In the presence of latents, the collection of MAGs $\MMM_1,\dots,\MMM_M$ are mixed graphs that also contain bidirected edges, which introduces issues, as bidirected edges cannot be detected easily. For example, two entities $i$ and $j$ might be such that $u \leftrightarrow v$ could be present in $\MMM_i$ and $u \leftarrow v $ could be present in $\MMM_j$, in which case intervening on just $u$ alone will not suffice to distinguish $i$ from $j$, we need interventions on both $u$ and $v$. This is the idea behind Algorithm~\BoundedDegree, which
identifies all the outgoing and bidirected edges incident on the set of sampled nodes (say $S$), for every entity in $[M]$. Since from this we can compute all neighboring relations of $u$ ($N_i(u)$), Algorithm~\BoundedDegree then checks whether these neighborhoods are the same or not for every node $u \in S$. We can now leverage the $\alpha$-clustering property to argue that this process succeeds with probability at least $1-\delta$.

As we use Algorithm~\Identifybidir, to find all bidirected edges incident on a node $u \in S$, we use an additional $O(\Delta)$ atomic interventions (per entity) where $\Delta = \max_{i \in [M], u \in V} \D{i}(u)$ is the maximum undirected degree in the PAGs $\PPP_1,\dots,\PPP_M$.

\begin{algorithm}[!ht]
\caption{\BoundedDegree}
\label{alg:alphaboundlatents}
\begin{algorithmic}[1]
\STATE \textbf{Input:} $\alpha > 0$, confidence parameter $\delta > 0$, PAGs $\UUU_1,\dots,\UUU_M$ of $M$ entities
\STATE \textbf{Output:} Partition of $[M]$ into clusters
\STATE Let $S$ denote a uniform sample of $\frac{2\log M/\delta}{\alpha}$ nodes from $V$ selected with replacement.
\FOR{every entity $i \in [M]$ and $u \in S$}
\STATE $\out{i}{u} \leftarrow \IdentifyNbr(i, u)$  \STATE $\bi{i}{u} \leftarrow \Identifybidir(i, u)$
\STATE $\inc{i}{u} \leftarrow \D{i}(u) \setminus \left( \out{i}{u} \cup \bi{i}{u} \right)$ 
\STATE Construct $\NN{i}{u}$ (defined in~\eqref{eqn:N})
\ENDFOR
\STATE Let $\mathcal{P}$ denote an empty graph on set of entities $[M]$
\FOR{every pair of entities $i, j$}
\IF{$\NN{i}{u} = \NN{j}{u}$ for every $u \in S$}
\STATE Include an edge between $i$ and $j$ in $\mathcal{P}$
\ENDIF
\ENDFOR
\STATE Return connected components in $\mathcal{P}$
\end{algorithmic}
\end{algorithm}

\begin{lemma}\label{lem:boundedlatent_app}
If the underlying MAGs $\MMM_1,\dots,\MMM_M$ satisfy $\alpha$-clustering property with true clusters $\Copt_1,\dots,\Copt_k$, then  Algorithm~\BoundedDegree exactly recovers the clusters $\Copt_1,\dots,\Copt_k$ with probability at least $1-\delta$. Every entity $i \in [M]$ uses at most $2 (\Delta+1) \log(M/\delta)/\alpha$ many atomic interventions.
\end{lemma}
\begin{proof}
We follow a proof idea similar to Lemma~\ref{lem:nolatent_app}. Again if two entities $i,j$ have different PAGs then they belong to different true clusters.

Consider two entities $i,j$ belonging to different true clusters but having the same PAG.
Again it suffices to find a node $u$ from the node-difference set $\diff(\MMM_i,\MMM_j) = \{ u \mid \NN{i}{u} \neq \NN{j}{u} \}$ to conclude that they belong to different clusters.

As there are latents (causal sufficiency), we cannot identify whether $u \in \diff(\MMM_i, \MMM_j)$, by checking only the outgoing neighbors of $u$ for entities $i, j$, and have to check the set of bidirected edges incident on $u$ as well. We can identify all the bidirected edges incident on $u$ for both $i, j$ using Algorithm~\Identifybidir.  Identifying such a node $u$, whose set of neighbors of node $u$ are different for entities $i, j$, provides a certificate that $i,j$ belong to different true clusters ($\alpha$-clustering property). In order to identify at least one node $u \in \diff(\MMM_i, \MMM_j)$, we use sampling.

Let $S$ denote the set of sampled nodes (with replacement) from $V$ such that $|S| = 2\log(M/\delta)/{\alpha}$. In Algorithm~\BoundedDegree, we construct interventional distributions for every node $u \in S$ and all the neighbors in the PAG given by $\DD{i}{u}$, for every entity $i \in [M]$. From these interventional distributions, we can compute $\NN{i}{u}$ and $\NN{j}{u}$ for all the nodes $u \in S$ (using Algorithms~\IdentifyNbr and~\Identifybidir).

From Claim~\ref{cl:separation}, we have that for every pair of entities $ i, j$ belonging to different true clusters, there exists $u \in S$ such that:
\[ \Pr[\NN{i}{u} \neq \NN{j}{u}] = \Pr[u \in \diff(\MMM_i, \MMM_j)] \geq 1-\delta. \]

Hence, with probability at least $1-\delta$, we can recover all the true clusters using Algorithm~\BoundedDegree. 

For an entity $i$, the total number of interventional distributions constructed is 
$$\sum_{u \in S}(1+|\DD{i}{u}|) \leq |S|(\Delta + 1) = 2 (\Delta+1) \log(M/\delta)/\alpha \ \text{ as } \max_{i \in [M], w \in V} |\DD{i}{w}| \leq \Delta.$$
\end{proof}

\begin{theorem}\label{thm:boundedlatent_app}
If the underlying MAGs $\MMM_1,\dots,\MMM_M$ satisfy $\alpha$-clustering property with true clusters $\Copt_1,\dots,\Copt_k$, then  Algorithm~\BoundedDegree exactly recovers these clusters with  probability at least $1-\delta$. Furthermore, if $\min_{b \in [k]} |\Copt_b| \geq n$, then there is an algorithm that exactly learns all these MAGs with probability at least $1-\delta$. Every entity $i \in [M]$ uses at most $2 (\Delta +1)\log(M/\delta)/\alpha + 1$ many atomic interventions.
\end{theorem}
\begin{proof}
From Lemma~\ref{lem:boundedlatent_app}, we can recover the clusters correctly with probability at least $1-\delta$. Using the {Meta-Algorithm} discussed in Section~\ref{app:meta}, and from Corollary~\ref{cor:cluster_to_graph_appendix}, we can obtain an algorithm to learn the graphs of every entity with an additional intervention per entity. This completes the proof.
\end{proof}

\subsection{Missing Details from Section~\ref{sec:graph_recovery_general}}\label{app:graph_recovery_general}

 In Algorithm~\GenRecovery, we obtain all the outgoing neighbors of the sampled set of nodes $S$. Then, we construct a graph on set of entities, $[M]$ such that an edge between a pair of entities $i, j$ is included if they share same PAGs, i.e.,  $\UUU_i =  \UUU_j$ and same outgoing neighbors for every node in $S$. However, it is possible that the graph $\mathcal{P}$ can contain more than one true cluster. In the next lemma, we show that we can detect this, and remove all the edges between entities belonging to two different clusters using $2$ interventions.
 
\begin{lemma}\label{lem:peeling}
 Suppose a component $T_a$ in $\mathcal{P}_\itr$ for some $\itr \geq 1$ contains all the entities from two true clusters $\Copt_b, \Copt_c$. If $\min_{r \in [k]} |\Copt_r| \geq \Omega(n \log M/\delta)$,  then, we can identify, with a probability $1-\delta/2 k^2$, all the pairs of entities $i', j' \in T_a$ such that $i' \in \Copt_b$ and $j' \in \Copt_c$ (or vice-versa) using at most $2$ interventions for every entity in $T_a$ .
\end{lemma}
\begin{proof}
We claim that if a component $T_a$ containing $\Copt_b$ and $\Copt_c$ exists, then, we can identify a pair of entities $i, j$ that are joined by an edge in $\mathcal{P}_\itr$ such that $i \in \Copt_b$ and $j \in \Copt_c$ or vice-versa. 

It suffices to find a node $u$ from the node-difference set $\diff(\MMM_i,\MMM_j) = \{ u \mid \NN{i}{u} \neq \NN{j}{u} \}$ to conclude that they belong to different clusters. From Claim~\ref{cl:separation}, we know that when $|S| = 2\log(2M/\delta)/\alpha$, we can identify such a $u \in \diff(\MMM_i, \MMM_j)$ with a probability $1-\delta/2$. We make the observation that a pair of entities $i, j$ that have an edge in this $\mathcal{P}_\itr$ and from different true clusters, can differ only if there is a node $u \in \diff(\MMM_i, \MMM_j)$ such that $u$ has a bidirected edge $u \leftrightarrow v$ in $\MMM_i$, and a directed edge $u \leftarrow v$ in $\MMM_j$ (or vice-versa). Intervening on both $u$ and $v$ will separate these entities, our main idea is to ensure that this happens. 

Consider a mapping $\pi : [M] \rightarrow V$ where $\pi(i)$ is assigned a node from $V$ selected uniformly at random. Using this mapping, we ensure that there are two entities $i \in T_a \cap \Copt_b, j \in T_a \cap \Copt_c$ joined by an edge, such that $\pi(i) = \pi(j) = v$ and $u \leftrightarrow v$ in $\MMM_i$, $u \leftarrow v$ in $\MMM_j$ (or vice-versa) for some $u \in S$. We have:

\begin{align*}
   \Pr[\text{ for any }i \in T_a, \ \pi(i) \neq v ] &=  1-\frac{1}{n}, \mbox{ and}\\
    \Pr[\forall i \in \Copt_b \ : \pi(i) \neq v] &= \left( 1-\frac{1}{n}\right)^{|\Copt_b|}.
\end{align*}
Similarly, we have \[  \Pr[\forall j \in \Copt_c \ : \pi(j) \neq v] = \left( 1-\frac{1}{n}\right)^{|\Copt_c|}. \]
\begin{align*}
    \Pr[\forall \ i \in \Copt_b, j \in \Copt_c \text{ such that } \pi(i) \neq v ,\ \pi(j) \neq v ] &= \left( 1- \frac{1}{n} \right)^{|\Copt_b| + |\Copt_c|} \\
    &\leq \frac{\delta}{2M^2} \leq \frac{\delta}{2k^2}\\
    \Rightarrow \Pr[\exists i \in \Copt_b, \exists j \in \Copt_c \ : \pi(i) = \pi(j) = v ] &\geq 1-\frac{\delta}{2k^2}.
\end{align*}

As we intervene on $\pi(i)$ for every entity $i \in T_a$, we know that there exists $i \in \Copt_b, j \in \Copt_c$, both in $T_a$ and that are assigned $v$ by $\pi$. Therefore, we can separate $i, j$ and remove the edge from $\mathcal{P}_\itr$. Now, we create an intervention on $\pi(i) = \pi(j) = v$ for every entity in $T_a$ and separate all the entity pairs $(i', j')$ joined by an edge in $\mathcal{P}_\itr$ that satisfy: $u \leftarrow v $ in $\MMM_{i'}$ and $u \leftrightarrow v $ in $\MMM_{j'}$ (or vice-versa). As we use at most two interventions for every entity in $T_a$, the lemma follows.
\end{proof}

\begin{lemma} \label{lem:latents_app}
If the underlying MAGs satisfy $\alpha$-clustering property with true clusters $\Copt_1,\dots,\Copt_k$ such that $\min_{b \in [k]} \Copt_b = \Omega(n \log (M/\delta))$ entities, the Algorithm~\GenRecovery exactly recovers the clusters $\Copt_1,\dots,\Copt_k$ with  probability at least $1-\delta$. Every entity $i \in [M]$ uses at most $O(\log (M/\delta)/\alpha+ k^2)$ many atomic interventions.
\end{lemma}
\begin{proof}
From Claim~\ref{cl:separation}, with probability at least $1-\delta/2$, we have that the set of sampled nodes $S$ (where $|S| = 2\log(2M/\delta)/\alpha$) satisfy that for every pair of entities from different clusters there is a node $u \in S$ that can be used to identify that they belong to different clusters. Using Lemma~\ref{lem:peeling}, we have that, in every iteration $\itr$, we remove all the edges in $\mathcal{P}_\itr$ between entities that are part of the same component but from different true clusters. After $k^2$ iterations, we would have separated all the pairs of entities between all the true clusters. In Algorithm~\GenRecovery, we return the connected components in $\mathcal{P}_\itr$ when there is no change in the set of edges between entities between $\mathcal{P}_{\itr-1}$ and $\mathcal{P}_\itr$. 

From Lemma~\ref{lem:peeling}, we have that, every entity performs at most $|S| + 2 k^2$ interventions. As there are at most $k^2$ iterations, and from Lemma~\ref{lem:peeling}, each iteration fails with probability at most $\delta/2k^2$, using union bound, we have that at least one of the iterations fails with probability at most $\delta/2$. 

Finally, using union bound for failure probability of calculating $S$ correctly, and failing in at least one of the iterations, we have, with probability at least $1-\delta$, Algorithm~\GenRecovery recovers the true clusters.
\end{proof}
From Lemma~\ref{lem:latents_app}, we know that we can recover the clusters correctly with probability at least $1-\delta$. Using the {Meta-Algorithm} discussed in Appendix~\ref{app:meta}, and from Corollary~\ref{cor:cluster_to_graph_appendix}, we can obtain an algorithm to learn the graphs of every entity with an additional intervention per entity. Combining it with guarantees obtained by Algorithm~\BoundedDegree in Theorem~\ref{thm:boundedlatent_app}, gives us the following result.

\begin{theorem}[Theorem~\ref{thm:latents} Restated]
If MAGs $\MMM_1,\dots,\MMM_M$ satisfy $\alpha$-clustering property with true clusters $\Copt_1,\dots,\Copt_k$ such that $\min_{b \in [k]} |\Copt_b| = \Omega(n \log (M/\delta))$. Then, there is an algorithm that exactly learns all these MAGs with probability at least $1-\delta$. Every entity $i \in [M]$ uses  
$\min \left\{ O(\Delta\log(M/\delta)/\alpha), O(\log(M/\delta)/\alpha + k^2) \right\}$ many atomic interventions.
\end{theorem}

\subsection{Lower Bound on the Number of Interventions}\label{app:lb}
In this section, we present a lower bound for the number of interventions required by every entity to recover true clusters. First, we state Yao's minimax theorem, which will be used to prove the lower bound.

\begin{theorem}[Yao's minimax theorem~\citep{yao1977probabilistic}]
Let $\mathcal X$ be a set of inputs to a problem and $\mathcal A$ the set of all possible deterministic algorithms that solve the problem. For any algorithm $A \in \mathcal{A}$ and $x \in \mathcal{X}$, let $\textrm{cost}(A, x)$ denote real-valued measure of cost of an algorithm $A$ on input $x$. Let $\nu, \mu$ be distributions over $\mathcal{A}$ and $\mathcal X$ respectively. Then, 

\begin{equation*}
    \max_{x \in \mathcal{X}} \E_{A \sim_\nu \mathcal{A}}[\textrm{cost}(A, x)] \geq \min_{a \in \mathcal{A}}\E_{X \sim_\mu \mathcal{X}}[\textrm{cost}(a, X)]
\end{equation*}\label{thm:yao}
\end{theorem}
Informally, the theorem states that to prove lower bounds on the cost of \emph{any} randomized algorithm, we have to find some distribution $\mu$ on inputs, such that \emph{every} deterministic algorithm $A \in \mathcal{A}$ has high cost. 

\begin{lemma}[Lemma~\ref{lem:veclb} restated]\label{lem:veclb_appendix}
Suppose we are given a vector $x \in \{0, 1\}^n$ with the promise that either $x = 0^n$ or $x$ contains $\alpha n$ ones. In order to distinguish these two cases with probability more than $2/3$, every randomized or deterministic algorithm must make at least $\Omega(1/\alpha)$ queries to the coordinates of the vector $x$.
\end{lemma}
\begin{proof}
It is easy to see that every deterministic algorithm for this problem requires $(1-\alpha)n+1$ queries. For obtaining a lower bound on the number of queries of any randomized algorithm, we use Yao's minimax theorem~\cite{yao1977probabilistic}. To do so, we construct an input distribution $\mu$ on $\{ 0, 1\}^n$ and show that every deterministic algorithm on the worst case requires at least $q$ queries while succeeding with a probability $2/3$. From Yao's minimax theorem (Thm~\ref{thm:yao}), this implies that every randomized algorithm requires at least $q$ queries to output the correct answer with probability of success $2/3$. We construct $\mu$ by using a probability of $1/2$ for $0^n$ vector and a probability of $1/2$ equally distributed among all vectors in $\{0,1\}^n$ containing exactly $\alpha n $ ones.

Suppose a deterministic algorithm (denoted by \ALG) is used to identify whether $x = 0^n$ or not. Let $\Ev(x)$ denote the event that the \ALG answers correctly on $x \in \{0, 1\}^n$, $Q(x)$ denote the set of queries used by \ALG such that $|Q(x)| = q$ and $L(x)$ denote the coordinates of $x$ that are non-zero.

Consider the event $\Ev(x)$ when \ALG answers correctly. We can write it as : 
\begin{align*}  &  \Pr_{x \sim \mu}[\Ev(x)] \\
& = \Pr[\Ev(x) \mid Q(x) \cap L(x) \neq \phi] \Pr[Q(x) \cap L(x) \neq \phi] + \Pr[\Ev(x) \mid Q(x) \cap L(x) = \phi] \Pr[Q(x) \cap L(x) = \phi]. \end{align*}
We calculate the probability that the coordinates queried are not part of the non-zero coordinates of $x$, given by $Q(x) \cap L(x) = \phi$ :
\begin{align*}
       & \Pr_{x \sim \mu}[Q(x) \cap L(x) = \phi] \\
       & = \Pr[Q(x) \cap L(x) = \phi \mid {x = 0^n}]\Pr[x = 0^n] + \Pr[Q(x) \cap L(x) = \phi \mid x \neq 0^n]\Pr[x \neq 0^n] \\
       &= \frac{1}{2}\left(\Pr[Q(x) \cap L(x) = \phi \mid {x = 0^n}] + \Pr[Q(x) \cap L(x) = \phi \mid x \neq 0^n] \right)\\
       &= \frac{1}{2}\left( 1 + \frac{{n-q \choose \alpha n}}{{n \choose \alpha n}} \right) = \frac{1+\tau}{2} \quad \text{, where } \tau = \frac{{n-q \choose \alpha n}}{{n \choose \alpha n}}.
\end{align*}

Now, we calculate the probability that \ALG answers correctly when the queries $Q(x)$ all return zero. We upper bound this probability by considering the case when \ALG answers `yes', and the case when \ALG answers `no' separately. It is easy to observe that $\Ev(x)$ is correct when \ALG=`yes' iff $x = 0^n$. Therefore, we have: 
\begin{align*}
    \Pr[\Ev(x) \mid Q(x) \cap L(x) = \phi] &\leq \max \left \{ \overbrace{\frac{\Pr[\Ev(x), Q(x) \cap L(x) = \phi]}{\Pr[Q(x) \cap L(x) = \phi]}}^{\text{ALG answers `yes'}}, \overbrace{\frac{\Pr[\Ev(x), Q(x) \cap L(x) = \phi]}{\Pr[Q(x) \cap L(x) = \phi]}}^{\text{ALG answers `no'}} \right \} \\ 
    &\leq \max \left \{  \frac{1/2}{(1+\tau)/2}, \frac{\tau/2}{(1+\tau)/2}\right \} \leq \frac{1}{1+\tau}.
\end{align*}

\begin{align*}
   \Pr_{x \sim \mu}[\Ev(x)] &= \Pr[\Ev(x) \mid Q(x) \cap L(x) \neq \phi] \Pr[Q(x) \cap L(x) \neq \phi] + \Pr[\Ev(x) \mid Q(x) \cap L(x) = \phi] \Pr[Q(x) \cap L(x) = \phi] \\
   &\leq \Pr[Q(x) \cap L(x) \neq \phi] + \Pr[\Ev(x) \mid Q(x) \cap L(x) = \phi] \Pr[Q(x) \cap L(x) = \phi] \\
                            &\leq \left(1 - \frac{1+\tau}{2} \right) + \frac{1}{1+\tau} \frac{1+\tau}{2} = 1-\frac{\tau}{2}.
\end{align*}
We know the probability of success for \ALG is at least $\frac{2}{3}$. Therefore, we have $\Pr_{x \sim \mu}[\Ev(x)] \geq \frac{2}{3}$, which implies $\tau \leq \frac{2}{3}$.

Let $H(x)$ denote the binary entropy function. Using the bound from (\cite{macwilliams1977theory}, Page 309)
\[
\sqrt{\frac{a}{8b(a-b)}}2^{a H(b/a)} \leq \binom{a}{b}\leq \sqrt{\frac{a}{2\pi b(a-b)}}2^{a H(b/a)}.
\]
We have
\begin{align*}
    \tau = \frac{{n-q \choose \alpha n}}{{n \choose \alpha n}} &\leq  \sqrt{\frac{8 (n-q)(n-\alpha n)}{2\pi n(n-q-\alpha n)}} 2^{(n-q)H(\frac{\alpha n}{n-q}) - n H(\alpha)}  \\ 
    &= \sqrt{\frac{4}{\pi}\left(1+\frac{q\alpha}{n-q-\alpha n }\right) }2^{(n-q)(H(\frac{\alpha n}{n-q}) - H(\alpha)) - qH(\alpha)}.
\end{align*}
We observe that $q \leq (1-\alpha) n + 1$ for any algorithm, as we can identify whether $x = 0^n$ or not trivially by querying more than $(1-\alpha) n + 1$ coordinates. Therefore, 
\begin{align*}
    \tau &\leq \sqrt{\frac{4}{\pi}\left(1-{q\alpha}\right)}2^{(n-q)(H(\frac{\alpha n}{n-q}) - H(\alpha)) - qH(\alpha)} \\
     &\leq \sqrt{\frac{4}{\pi}} 2^{-q\alpha \log e/2 + (n-q)(H(\frac{\alpha n}{n-q}) - H(\alpha)) - qH(\alpha)}
\end{align*}
Using $\frac{\alpha n}{n-q} \geq \alpha$ and mean-value theorem, we have:
\begin{align*}
  (n-q)  \left(H\left(\frac{\alpha n}{n-q} \right) - H(\alpha) \right) &\leq q \alpha H'\left(\frac{\alpha n}{n-q}\right)\\
  &= q \alpha \log \left(\frac{n-q}{\alpha n}-1\right)\\
  &\leq q \alpha \log({1-\alpha}) - q\alpha \log{\alpha}.
\end{align*}
Substituting the above expression and expanding $H(\alpha)$, we have :
\begin{align*}
    \tau &\leq \sqrt{\frac{4}{\pi}} 2^{-q\alpha \log e/2 + q \alpha \log({1-\alpha}) - q\alpha \log{\alpha} + q\alpha \log\alpha +(1-\alpha)q\log(1-\alpha)}\\
    &\leq \sqrt{\frac{4}{\pi}} 2^{-q\alpha \log e/2 + q\log(1-\alpha)}\\
    &\leq \sqrt{\frac{4}{\pi}} 2^{-q\alpha \log e/2 - q\alpha} \leq \frac{2}{3}.
\end{align*}

Therefore, for \ALG to succeed with probability at least $2/3$, we have $$q \geq \Omega\left(\frac{1}{\alpha} \right).$$

Using this with Yao's minimax theorem (Thm~\ref{thm:yao}), we get that with every randomized algorithm needs $\Omega(1/\alpha)$ queries to succeed on this problem with probability at least $2/3$.
\end{proof}

\section{Experimental Evaluation} \label{app:expts}

In this section, we provide additional details about the experimental evaluation discussed in Section~\ref{sec:expts}. 

\subsection{Learning MAGs under $\alpha$-clustering property}

\begin{table*}[t]
\centering
\hspace*{-.3cm}
\small	 
  \begin{tabular}{|l|l|l|l|l|l|l|c|}
    \toprule
      Causal &
      \multicolumn{3}{c|}{FCI} &
      \multicolumn{3}{c|}{\BoundedDegree(Alg.~\ref{alg:alphaboundlatents})} &
      \multicolumn{1}{c|}{Maximum} \\
     Network &  {Precision} & {Recall} & {Accuracy} & {Precision} & {Recall} & {Accuracy} & {\# Interventions}\\
      \midrule
    \textit{Earthquake} & $0.79\pm0.25$ & $0.98\pm 0.02$ &$0.79\pm0.25 $  &$1\pm0.00 $  & $1.0\pm0.0 $ & $1.00\pm0.00 $ & 3  \\
    \textit{Survey} & $0.79\pm0.25$ & $0.98\pm 0.02$ &$0.79\pm0.25 $ & $0.89\pm 0.20$ & $1.0\pm 0.00$ & $0.89 \pm 0.20$ & 4\\
    \textit{Asia} & $0.84 \pm 0.23$ & $0.98\pm 0.02$ &$0.84\pm0.23 $  & $0.89\pm 0.20$ & $1.0\pm 0.00$ & $0.89 \pm 0.20$ & 4 \\
    \textit{Sachs} & $1.0 \pm 0.00$ & $1.0\pm 0.00$ & $1.0\pm 0.00$ & $0.79\pm 0.25$ & $1.0 \pm 0.00$& $0.79\pm 0.25$ & 5\\
    \textit{Erd\"os-R\'enyi} & $1.00\pm0.00$ & $1.00\pm0.00$ & $1.00\pm0.00$ & $1.00\pm0.00$ & $1.00\pm0.00$ & $1.00\pm0.00$ & 5 \\
    \bottomrule
  \end{tabular}
  \caption{In this table, we present the precision, recall and accuracy values obtained by Algorithm~\BoundedDegree and FCI. Each cell includes the mean value along with the standard deviation computed over 10 runs. The last column contains the maximum number of interventions per entity required (including both Algorithm~\BoundedDegree and the Meta-algorithm) for recovering the DAGs.}
  \label{table:alpha}
\end{table*}

\smallskip
\noindent\textbf{(Synthetic) Data Generation.} We use following process for each of the five considered causal network (\textit{Asia}, \textit{Earthquake}, \textit{Sachs}, \textit{Survey}, and \textit{Erd\H{o}s-Renyi}).  We construct causal DAGs for $M$ entities distributed among the clusters $\Copt_1, \Copt_2, \cdots \Copt_k$ equally, i.e., $|\Copt_i| = M/k$ for all $i \in [k]$. Again we set $k=2$. For each of the DAGs $\DDD_1$ and $\DDD_2$, we select a certain number of pairs of nodes randomly, and include a latent variable between them, that has a causal edge to both the nodes. In our experiments, we used 2 latents per DAG. This results in two new DAGs $\DDD'_1$ and $\DDD'_2$. To ensure $\alpha n$ node distance between clusters, we modify $\DDD'_2$ using random changes until the two DAGs $\DDD'_1$ and $\DDD'_2$ are separated by a distance of $\alpha n$ and are Markov equivalent. Without Markov equivalence, we observe that FCI always recovers the underlying clusters correctly in the $\alpha$-clustering case.\footnote{Again this is not true for $(\alpha,\beta)$-clustering, as shown by our experiments results for that case, because now difference in PAGs between two entities does not automatically imply that those two entities must belong to different clusters.}  However, existence of Markov equivalent DAGs is a well-known problem in real-world graphs, a popular example to illustrate this comes the `` breathing dysfunction'' causal graph in Fig.\ 3 in ~\citep{zhang2008completeness}. We create $M/2$ copies of the each of the two DAGs $\DDD'_1$ and $\DDD'_2$ and assign it to distinct entities in each of the two clusters.

\smallskip
\noindent \textbf{Parameters.} We present the following settings for the model parameters, $\alpha$ is at least $0.60$, $M = 40$. For the synthetic data generated using Erd\"os-R\'enyi model, we use $n = 10$, probability of edge $p = 0.30$. We ran all of our experiments for 10 times with the stated values and report the results.

\smallskip
\noindent \textbf{Sample Set $S$ Size.} For Algorithm \BoundedDegree, we again tried different set $S$ sizes ranging from $1$ to $3$. In Figure~\ref{fig:int_vs_sample_a}, we plot the mean value of the maximum number of interventions per entity with increase in sample size. It has a same trend as with $(\alpha,\beta)$-clustering (Figure~\ref{fig:int_vs_sample_a}). A sample size of  $1$ roughly corresponds to around $3$ interventions, and we use that for the results presented in Table~\ref{table:alpha}.

\begin{figure}[H]
\centering
    \includegraphics[scale=0.65]{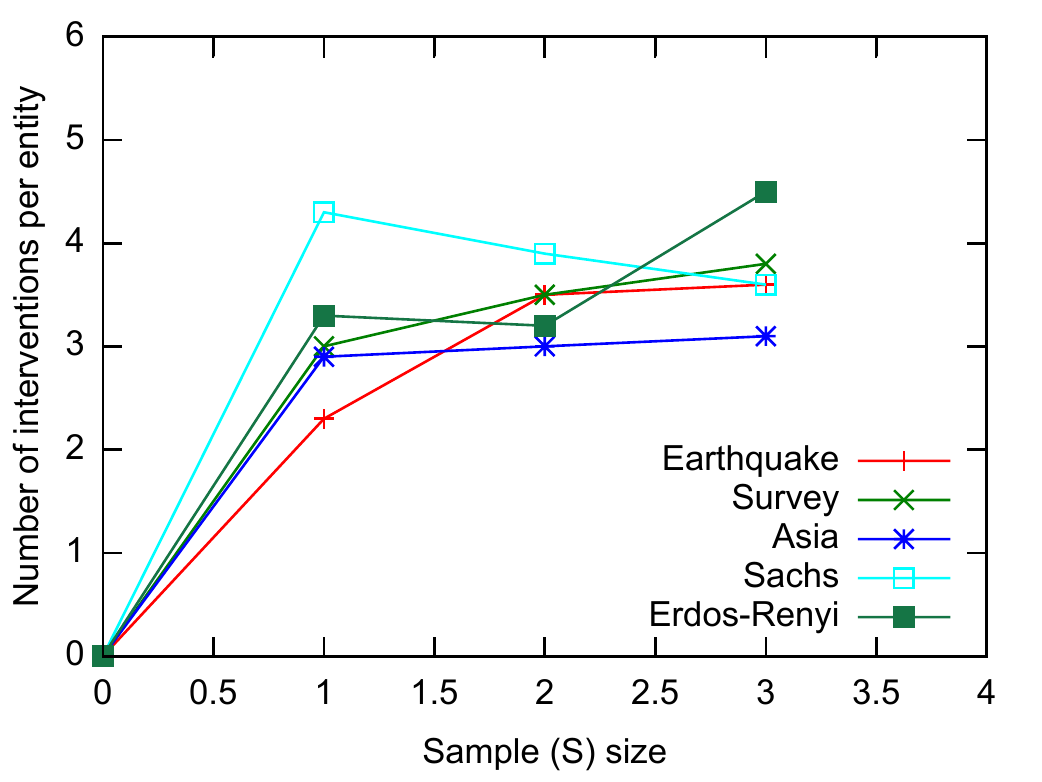}
    \caption{Sample size vs.\ maximum number of interventions per entity used by Algorithm~\BoundedDegree.}
    \label{fig:int_vs_sample_a}
\end{figure}

\smallskip
\noindent \textbf{Evaluation of Clustering.} We start by results on recovering the clustering using Algorithm~\BoundedDegree. As a baseline, we again employ the well-studied FCI algorithm~\citep{spirtes2000causation}. After recovering the PAGs corresponding to the DAGs using FCI, we cluster them by constructing a similarity graph (similar to the case of $(\alpha, \beta)$-clustering discussed previously) defined on the set of entities. For Algorithm~\BoundedDegree, we first construct a sample $S$, and perform various interventions based on the set $S$ for every entity to finally obtain the $k$ clusters. We also implemented another baseline algorithm (\textsc{Greedy}) that uses interventions, based on a greedy idea that selects nodes to set $S$ in Algorithm~\BoundedDegree by considering nodes in increasing order of their degree in the PAGs returned by FCI. We use this ordering to minimize the number of interventions as we intervene on every node in $S$ and their neighbors. We use the same metrics as the $(\alpha,\beta)$-clustering case.

\smallskip
\noindent \textbf{Results.} 
In Table~\ref{table:alpha}, we compare Algorithm~\BoundedDegree to FCI on the clustering results. For Algorithm~\BoundedDegree, we use a sample $S$ of size $1$, and observe in Figure~\ref{fig:int_vs_sample_a}, that this corresponds to about $2$ interventions per entity. With increase in sample size, we observed that the results were either comparable or better. We observe that our approach leads to considerably better performance in terms of the accuracy metric with an average difference in mean accuracy of about $0.20$. We observe that entities belonging to the same true cluster are always assigned to the same cluster, resulting in high recall for both Algorithm~\BoundedDegree and FCI. Further, the higher value of precision for our algorithm is because FCI is unable to correctly detect that there are two clusters, as the DAGs are Markov Equivalent which means that they result in the same PAGs. 

Algorithm~\BoundedDegree outperforms the \textsc{Greedy} baseline for the same sample ($S$) size. For example, on the \textit{Earthquake} and \textit{Survey} causal networks, Algorithm~\BoundedDegree obtains the mean accuracy values of $1.0$ and $0.89$ respectively, while  \textsc{Greedy} for the same number of interventions obtained an accuracy of only $0.74$ and $0.64$ respectively. On the remaining causal networks, the accuracy values of \textsc{Greedy} are almost comparable to our Algorithm~\BoundedDegree.

After clustering, we recover the DAGs using the Meta-algorithm described in Section~\ref{app:meta}, and observe that only one additional intervention is needed. In the last column in Table~\ref{table:alpha}, we report the maximum number of interventions for recovering DAGs, which includes both the interventions used by the Algorithm~\BoundedDegree and the Meta-algorithm. We observe that our \emph{collaborative} approach uses fewer interventions for MAG recovery compared to the number of nodes in each causal network. For example, in the Erd\"os-R\'enyi setup, the number of nodes $n=10$, whereas we use at most $5$ interventions per entity. Thus, compared to the worst-case, cutting the number of interventions for each entity by $50\%$.

\end{document}